\documentclass{article}
\usepackage{iclr2027_conference,times}
\usepackage[utf8]{inputenc}
\usepackage[T1]{fontenc}
\usepackage{hyperref,url}
\usepackage{booktabs,amsmath,amsfonts,amsthm}
\usepackage{graphicx,microtype,xcolor,enumitem}

\newtheorem{proposition}{Proposition}

\hypersetup{hidelinks,pdftitle={Context-aware tokenization for Cross-subject Emotion Decoding from EEG},pdfauthor={Jiaxin Qing, Lexin Li}}
\title{Context-aware tokenization for \\ Cross-subject Emotion Decoding from EEG}
\author{Jiaxin Qing \\
University of California, Berkeley \\
\texttt{jxqing@berkeley.edu} \And
Lexin Li \\
University of California, Berkeley \\
\texttt{lexinli@berkeley.edu}}

\iclrfinalcopy
\begin{document}

\maketitle

\vspace{-0.1in}
\begin{abstract}
\vspace{-0.1in}
In cross-subject EEG emotion decoding, neighboring windows can inform the representation of a target window, but they contain both background variation and sustained task signals. This creates a challenge for contextual representation learning because subtracting shared activity can also remove useful information. We introduce the Morlet Spectral Transformer (MST), which conditions target-window tokens on a structured spectral summary before attention. MST averages the Morlet log-amplitude spectra of unlabeled same-trial neighbors and uses frequency-specific electrode projections to combine the target spectrum with its reference-relative residual. This retains access to absolute activity while incorporating context into a fixed attended target-token grid. Without external pretraining, MST achieves 66.5\%, 40.7\%, 36.1\%, and 27.9\% accuracy on SEED, SEED-IV, SEED-V, and SEED-VII under leave-one-subject-out evaluation, averaged over three random seeds, outperforming the evaluated pretrained and from-scratch baselines. On FACED, MST achieves 17.52\% accuracy in nine-class cross-subject evaluation, exceeding the strongest evaluated baseline by 1.15 percentage points. We also performed the ablation studies across all 15 SEED subjects to evaluate the contribution of different components of MST.

\end{abstract}

\section{Introduction}

Cross-subject emotion recognition from electroencephalography (EEG) aims to decode affective states in previously unseen individuals \citep{zheng2015investigating}. A short EEG window is observed within a longer recording: its spectral activity reflects both the ongoing task and the person's recording background. The representation problem is therefore not only how to encode a window, but how to use its surroundings when constructing that encoding.

Contextual representation learning addresses a related problem in retrieval: document and chunk embeddings can incorporate neighboring information instead of representing each unit in isolation \citep{morris2024contextual,gunther2024late}. For EEG, however, local context has a dual role. Adjacent windows share spectral structure that can include both nuisance variation and sustained emotion-related activity. Subtracting a local reference exposes deviations but can also remove useful task signals. A contextual token should therefore retain access to absolute spectral activity as well as its relation to the surrounding recording.

EEG foundation models learn reusable representations through large-scale pretraining \citep{jiang2024labram,wang2024eegpt,wang2025cbramod,zhou2025csbrain}. Temporal context and spectral comparison also have precedents in emotion recognition and seizure detection \citep{ding2025emt,tu2024dmnet}. We study a specific design question: \emph{how can neighboring EEG windows condition a target representation before attention, without increasing the number of attended tokens?}

We address this question with context-aware tokenization in the Morlet Spectral Transformer (MST). A fixed Morlet bank places the target and its neighbors on a common time-frequency grid. A fixed average of the neighboring log-amplitude spectra forms a local reference. Learned frequency-specific electrode projections then combine the target spectrum with its reference-relative residual. Context thus enters through a structured spectral summary, while attention operates on the target grid. This joint construction makes local context available within each target token before attention.

We evaluate MST on four SEED-family datasets and FACED, covering three to nine emotion classes. Experiments use strict leave-one-subject-out evaluation on the SEED family and ten subject-disjoint folds on FACED, with results averaged over three training seeds. MST achieves the highest accuracy among the evaluated models on all five datasets. Component ablations on all 15 SEED subjects, matched-span context comparisons, and a past-only reference further quantify the contributions of the spectral representation, spatial projection, and neighboring context.

Our contributions are:
\begin{itemize}[topsep=0pt, itemsep=1pt, parsep=0pt, partopsep=0pt]
\item We formulate local context use as conditioning target-window tokens on a neighboring spectral summary before attention, retaining a fixed target token grid.
\item We instantiate this design with Morlet spectra and frequency-specific electrode projections over absolute and residual views. This parameterization preserves access to both the target and its reference, allowing the model to combine absolute activity and local deviations.
\item Without external pretraining, MST improves accuracy over the strongest evaluated pretrained baseline on each of four SEED-family datasets by 3.15--4.15 percentage points. All-subject SEED ablations and context controls characterize the roles of neighboring information and frequency-specific projection.
\end{itemize}

\section{Related Work}

\vspace{-0.1in}
\paragraph{EEG emotion recognition and contextual modeling.}
Differential entropy (DE) is an established EEG emotion feature \citep{duan2013differential,zheng2015investigating}; IAG learns instance-adaptive graphs over five-band DE \citep{song2020iag}. Reference-based methods use dedicated resting recordings \citep{ahmed2023baseline} or, in JDAT, pre-trial subtraction before wavelet encoding on DEAP/DREAMER \citep{wang2021jdat}. JDAT's SEED experiment uses no baseline correction and evaluates within subjects. EmT models temporal context through a sequence of graph-derived tokens \citep{ding2025emt}, and DMNet uses contextual spectral differences for seizure detection \citep{tu2024dmnet}. MST summarizes same-trial neighbors before attention and retains both absolute and reference-relative spectra. Appendix~\ref{app:closest-work} compares the closest mechanisms and protocols.

\vspace{-0.1in}
\paragraph{EEG foundation models and spectral tokenizers.}
LaBraM, EEGPT, CBraMod, and CSBrain learn representations through large-scale pretraining \citep{jiang2024labram,wang2024eegpt,wang2025cbramod,zhou2025csbrain}. BIOT and BrainBERT use Fourier or STFT features, while TFM-Tokenizer and CodeBrain learn time-frequency codebooks \citep{yang2023biot,wang2023brainbert,pradeepkumar2026tfm,ma2026codebrain}. A recent benchmark finds that from-scratch models remain competitive and larger pretraining scales do not consistently improve cross-subject generalization \citep{liu2026eegfm}. MST studies a complementary choice: continuous tokens conditioned on a local spectral reference, learned without external pretraining.

\vspace{-0.1in}
\paragraph{Time-frequency representations and spatial projection.}
Morlet wavelets provide frequency-dependent resolution for EEG analysis \citep{tallonbaudry1997oscillatory,cohen2014analyzing}. Analformer also uses fixed Morlet kernels in Transformer patch embedding \citep{yeom2026analformer}. Frequency-dependent spatial filtering has precedents in FBCSP, EEGNet, and FBCNet \citep{ang2008filter,lawhern2018eegnet,mane2021fbcnet}. MST's distinction lies in jointly projecting absolute and locally referenced spectra separately at each frequency, making a neighboring summary available in each target token.

\vspace{-0.1in}
\paragraph{Contextual representations and context engineering.}
Contextual document embeddings encode neighboring documents \citep{morris2024contextual}; late chunking uses long-context attention before pooling chunk embeddings \citep{gunther2024late}. MST shares the aim of representing a local unit using its surroundings, but pools neighboring spectra before attention. Agent context engineering concerns the selection and organization of information available during inference \citep{anthropic2025context}. At the perception stage, MST organizes local EEG context into a fixed spectral summary that conditions target representations. Integration with agent decision systems is a future application.

\section{Method}
\label{sec:method}

MST constructs a contextual representation in three steps: spectral encoding, local context summarization, and target conditioning (Figure~\ref{fig:architecture}). Target and neighboring windows are mapped to the same Morlet grid; a fixed average summarizes the neighbors; frequency-specific projections combine the target with its reference-relative residual. The Transformer classifies only the resulting target-grid tokens. At inference, observed context changes token values without updating model parameters.

\begin{figure}[t]
\centering
\includegraphics[width=\textwidth,height=3in]{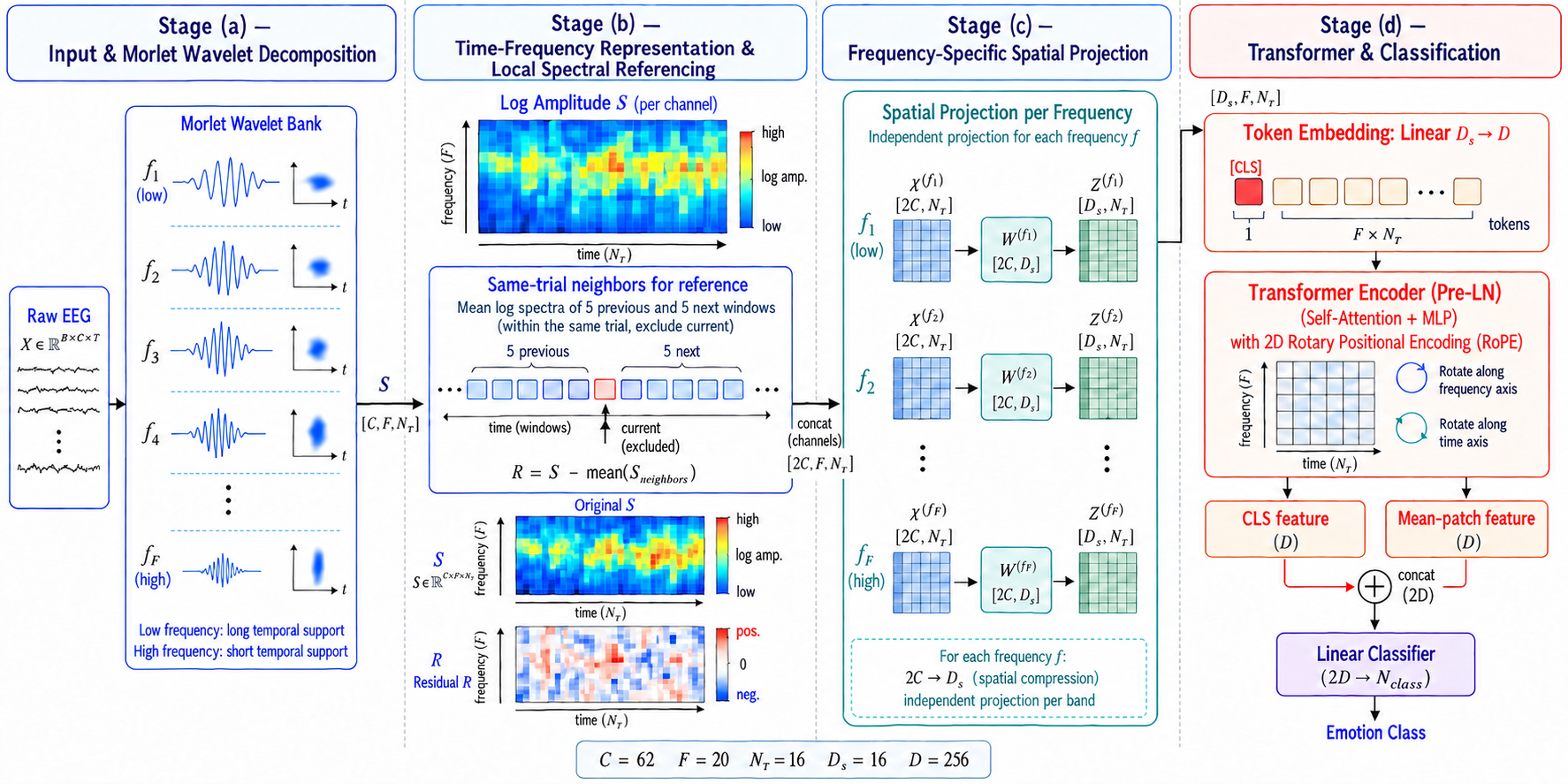}
\par\smallskip
{\footnotesize Same-trial neighbors $\rightarrow$ spectral reference $\rightarrow$ target tokens (320 patches on SEED).\par}
\caption{Context-aware tokenization in MST. Morlet decomposition produces time-frequency spectra. The target's original spectrum and its deviation from the neighboring-window reference are combined by frequency-specific spatial projections, then embedded as target-window tokens. Neighboring windows enter through the reference without extending the Transformer sequence. A Transformer with 2D RoPE and a dual-pooling head predicts the emotion label. Heatmaps are schematic.}
\label{fig:architecture}
\end{figure}

\subsection{Morlet Wavelet Tokenization}
\label{sec:wavelet}

Emotion-related activity spans multiple temporal scales; a Morlet grid represents these scales before context aggregation.

We adopt the complex Morlet wavelet, $\psi_f(t) = \exp(-t^2/(2\sigma_f^2))\exp(j\, 2\pi f\, t)$, a complex sinusoid with center frequency $f$ windowed by a Gaussian envelope of width $\sigma_f \propto 1/f$ \citep{tallonbaudry1997oscillatory, cohen2014analyzing}.  The fixed bank uses longer windows at lower frequencies and shorter windows at higher frequencies, providing a frequency-dependent trade-off in temporal and spectral resolution \citep{cohen2019morlet}. In the SEED-family configuration, we place $F = 20$ center frequencies on a logarithmic grid from $2$ to $45$\,Hz, covering delta through gamma frequencies. We then convolve each channel with the bank, take the complex amplitude, adaptively pool into $N_\tau = 16$ temporal bins, and apply a log transform, which produces a tensor $\mathbf{S} \in \mathbb{R}^{C \times F \times N_\tau}$ with $C = 62$. The grid formula and visualization of the wavelet bank and resulting tokenization are given in Appendix~\ref{app:morlet-bank}. 

This grid extends the use of log spectral features from segment-level summaries to a dense time-frequency representation. The following proposition connects Morlet tokens to Gaussian differential entropy, with assumptions and proof in Appendix~\ref{app:de-equivalence}.

\begin{proposition}[Morlet tokens and differential entropy, informal]
\label{thm:morlet-de-informal}
For a zero-mean proper complex Gaussian Morlet response, the log mean envelope equals the differential entropy of its real quadrature up to a universal additive constant. Empirical tokens approach this statistic under ergodic averaging and negligible numerical stabilizers.
\end{proposition}

The $F\times N_\tau=320$ grid samples local spectral variation more densely than a single DE value per canonical band. Its frequency-dependent passbands and temporal pooling retain within-segment information.

\subsection{Local Context Summarization by Spectral Referencing}
\label{sec:baseline}

Adjacent two-second EEG segments share spectral structure through temporal autocorrelation. This structure can include recording background, subject-specific activity, and sustained emotion-related signals. Figure~\ref{fig:redundancy} illustrates the similarity of neighboring log-amplitude spectrograms and the effect of local spectral referencing.

\begin{figure}[t]
\centering
\includegraphics[width=\textwidth]{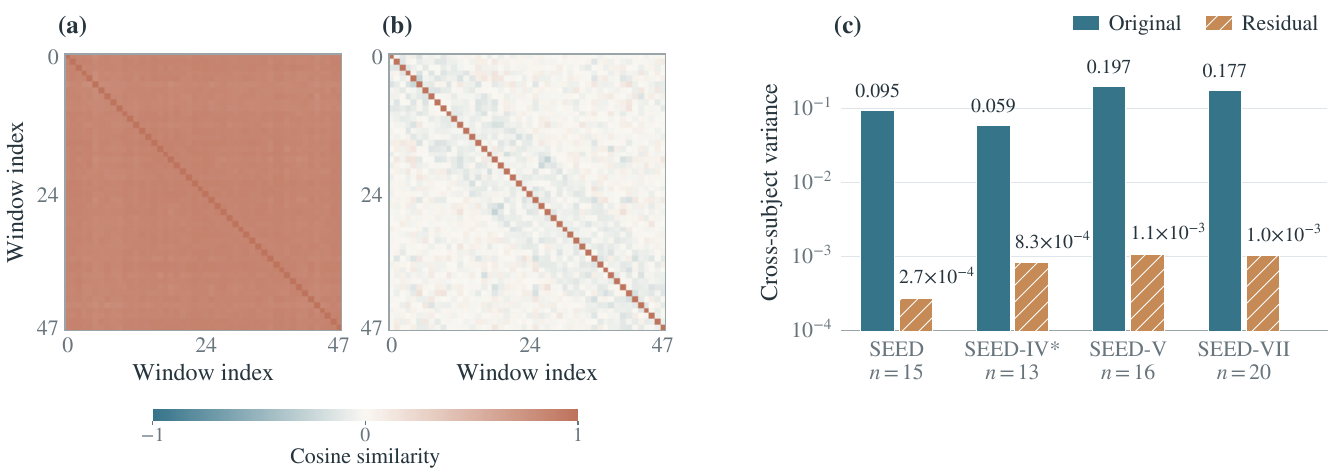}
\caption{(a,b) Cosine similarities of 48 consecutive SEED windows before and after local referencing. (c) Variance across subject-mean spectra, using 256 windows per subject. All panels use $K=10$ neighbors per side. Appendix~\ref{app:spectral-diagnostic} gives sampling and full-cohort results.}
\label{fig:redundancy}
\end{figure}

We make token construction context-aware by estimating a local spectral reference. For each segment, we average the log-amplitude spectrograms of the $K$ preceding and $K$ subsequent segments within the same trial, i.e., 
\begin{equation*} %
\bar{S}_{c,f,\tau} = \frac{1}{2K} \sum_{i=1}^{2K} S^{(i)}_{c,f,\tau},
\end{equation*}
with $K = 5$ giving 20 seconds of neighboring context for the two-second SEED windows. We then form the residual,
\begin{equation*} \label{eq:residual}
S^{\mathrm{res}}_{c,f,\tau} = S_{c,f,\tau} - \bar{S}_{c,f,\tau}.
\end{equation*}
The residual suppresses spectral components shared across the neighboring segments, emphasizing deviations from the local context. To preserve spectral information removed by the residual, we concatenate the two branches along the channel dimension as $\tilde{\mathbf{S}} = [\mathbf{S},\; \mathbf{S}^{\mathrm{res}}] \in \mathbb{R}^{2C \times F \times N_\tau}$, preserving absolute spectral information alongside the locally referenced view. Missing neighbors are zero EEG windows before the Morlet transform, with the denominator fixed at $2K$; context stays within each trial. A past-only variant uses ten preceding windows. The residual is a log ratio to the geometric mean of the stabilized neighboring envelopes. A shared amplitude gain cancels in the residual for complete neighborhoods when numerical stabilizers are negligible; the retained absolute branch does not share this invariance. Appendix~\ref{app:context} derives this relationship and its boundary behavior.

Figure~\ref{fig:redundancy} shows how referencing changes these shared spectral components. Mean off-diagonal cosine similarity falls from 0.850 to $-0.020$ in the illustrated trial, while subject-mean spectral dispersion decreases across all four diagnostic cohorts. These statistics characterize shared spectral structure, which can include both background and sustained task activity. Section~\ref{sec:ablations} evaluates its contribution to decoding.

\subsection{Frequency-Specific Spatial Projection}
\label{sec:spatial}

EEG rhythms have frequency-dependent topographies \citep{cohen2014analyzing}. Separate projections allow electrode combinations to vary with frequency when integrating target and context information.

\begin{figure}[t]
\centering
\includegraphics[width=0.8\textwidth]{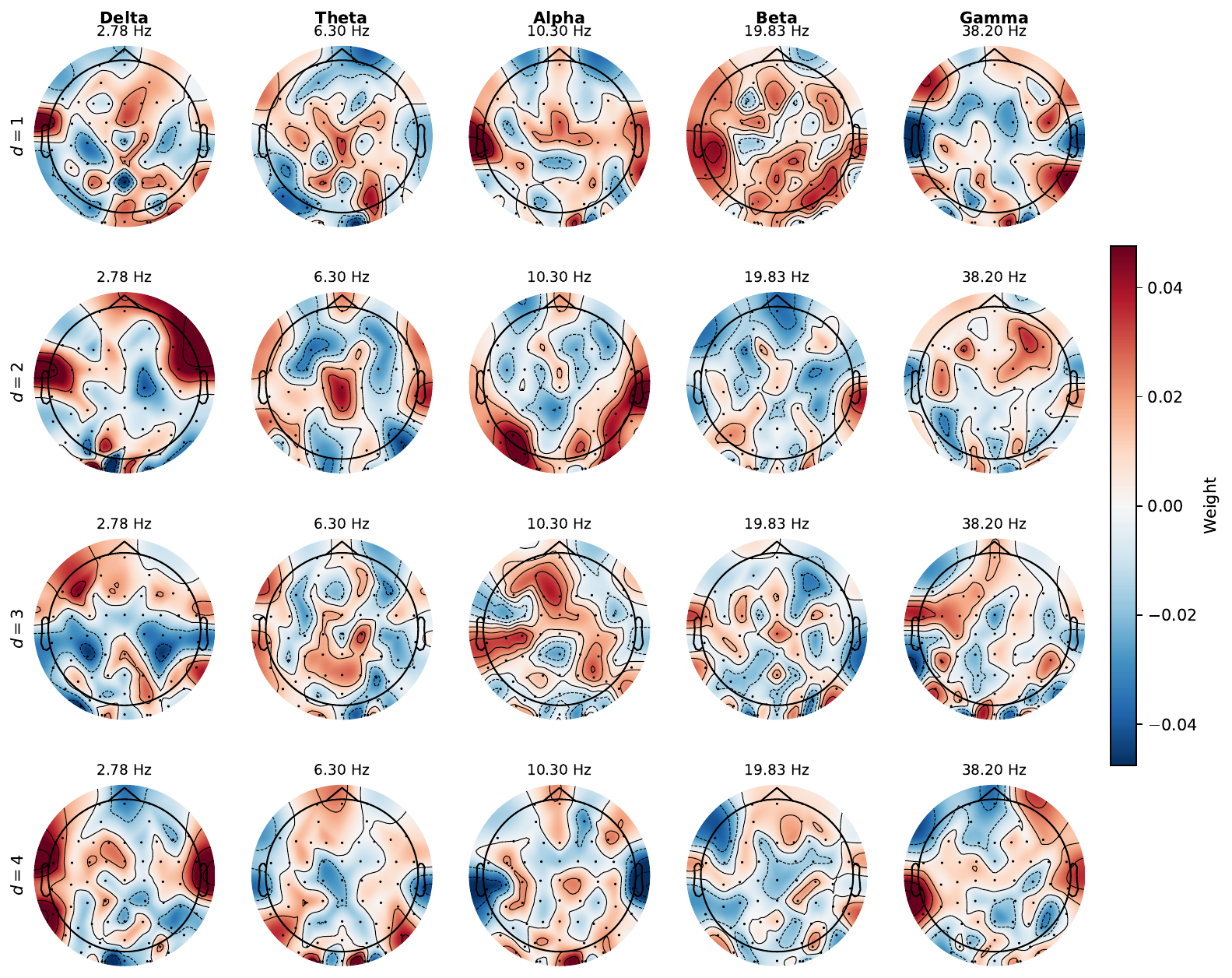}
\caption{Frequency-specific spatial projection weights $W^{(f)}_{1:C,d}$ for the original-spectrum branch. Rows show the first four components in index order and columns five representative frequencies. Colors encode signed electrode weights. Appendix~\ref{app:visualizations} describes component selection and plotting.}
\label{fig:spatial_filters}
\end{figure}

We therefore learn a separate projection matrix for each sampled frequency. Let $\mathbf{W}^{(f)} \in \mathbb{R}^{2C \times D_s}$ be the projection for frequency $f$, with $D_s = 16$. We compute the spatial projection as
\begin{equation*} %
Z_{d,f,\tau} = \sum_{c=1}^{2C} W^{(f)}_{c,d}\, \tilde{S}_{c,f,\tau}.
\end{equation*}
The output is $\mathbf{Z} \in \mathbb{R}^{D_s \times F \times N_\tau}$. 
Figure~\ref{fig:spatial_filters} displays the learned electrode weights as scalp topographies, making the frequency-dependent spatial combinations directly inspectable.
Partitioning $W^{(f)}$ into original- and residual-branch blocks, $W_{\mathrm{abs}}^{(f)},W_{\mathrm{res}}^{(f)}\in\mathbb{R}^{C\times D_s}$, makes the reference contribution explicit,
\begin{equation}
Z_{:,f,\tau}
=\bigl(W_{\mathrm{abs}}^{(f)}+W_{\mathrm{res}}^{(f)}\bigr)^\top S_{:,f,\tau}
-\bigl(W_{\mathrm{res}}^{(f)}\bigr)^\top\bar S_{:,f,\tau}.
\label{eq:reference-token}
\end{equation}
Equation~\eqref{eq:reference-token} shows that the absolute--residual coordinates are a linear reparameterization of jointly projecting the target and reference: setting $A^{(f)}=W_{\mathrm{abs}}^{(f)}+W_{\mathrm{res}}^{(f)}$ and $B^{(f)}=-W_{\mathrm{res}}^{(f)}$ gives $Z=A^\top S+B^\top\bar S$. The two parameterizations therefore span the same linear maps. The absolute--residual coordinates organize absolute activity and local deviations in each token while preserving access to shared information. The reference changes token values while the SEED Transformer retains 320 patch tokens.

\subsection{What Context Contributes to a Target Token}
\label{sec:context-design}

\paragraph{A structured summary of the surrounding recording.}
The reference retains electrode, frequency, and within-window time-bin axes while averaging over neighboring windows. The pair $[S,S-\bar S]$ retains both inputs before spatial compression, since $\bar S=S-(S-\bar S)$. Equation~\eqref{eq:reference-token} permits different uses of context through learned coefficients. Setting $W_{\mathrm{res}}=0$ recovers a target-only projection, while $W_{\mathrm{abs}}=0$ yields a residual-only projection. Intermediate settings combine absolute activity and local deviations, allowing shared activity to contribute task evidence. Averaging summarizes recurring spectral content across neighbors, pooling over their order and variability.

\paragraph{A learned response to local context.}
For a fixed target spectrum, changing the reference by $\Delta\bar S$ changes the projected features by $\Delta Z_{:,f,\tau}=-(W_{\mathrm{res}}^{(f)})^\top\Delta\bar S_{:,f,\tau}$. Thus the same target window can receive different tokens as its recording context changes. The averaging rule is fixed, but its effect on each token is learned separately at each frequency through $W_{\mathrm{res}}^{(f)}$. Context conditioning occurs in a forward pass with fixed model parameters, making local recording structure available to the Transformer before prediction.

\paragraph{Context enters before prediction.}
For $M$ neighbors and $N=FN_\tau$ patches per window, concatenating equally resolved target and neighbor grids would supply $(M+1)N$ patch tokens to attention. MST instead conditions $N$ target patches, followed by one CLS token. Context therefore influences nonlinear feature interactions before classification, whereas probability averaging combines completed predictions. Extending the reference span keeps the attended grid fixed but adds spectral computation.

\subsection{Transformer Encoder and Classification}
\label{sec:transformer}

We project each $(f, \tau)$ bin of the spatially projected grid linearly into a token of dimension $D = 256$,
\begin{equation*} %
\mathbf{h}^{(0)}_{f \cdot N_\tau + \tau} = \mathrm{Linear}(\mathbf{Z}_{:,f,\tau}) \in \mathbb{R}^{D},
\end{equation*}
yielding 320 patch tokens on SEED. We prepend a learnable CLS token and apply pre-LayerNorm Transformer blocks \citep{vaswani2017attention} with two-dimensional rotary position embedding (2D RoPE) over the frequency--time grid \citep{su2024roformer}. 

After the Transformer layers, we concatenate the CLS token $\mathbf{z}_{\mathrm{CLS}}\in\mathbb{R}^D$ and the mean of all patch tokens $\bar{\mathbf{z}}_{\mathrm{patch}}=(FN_\tau)^{-1}\sum_i\mathbf{h}^{(L)}_i$. A linear head classifies the resulting $2D$-dimensional vector,
\begin{equation*} %
\hat{\mathbf{y}} = \mathbf{W}_{\mathrm{cls}}\, [\mathbf{z}_{\mathrm{CLS}},\; \bar{\mathbf{z}}_{\mathrm{patch}}] + \mathbf{b}_{\mathrm{cls}},
\end{equation*}
where $\mathbf{W}_{\mathrm{cls}} \in \mathbb{R}^{N_{\mathrm{class}} \times 2D}$ and $N_{\mathrm{class}}$ is the number of emotion categories of the target dataset.

\subsection{Data Augmentation}
\label{sec:augmentation}

Training a Transformer on small EEG datasets is prone to overfitting. We therefore augment the training data with four strategies. Phase perturbation adds noise to Fourier phases. Band-specific noise injection varies local spectral content. Channel dropout masks electrodes and replaces them with learned wavelet-domain tokens. Wavelet time roll circularly shifts the log-amplitude grid along time, encouraging invariance to within-window position. These augmentations are used only during training. Their parameters and an illustration appear in Appendices~\ref{app:datasets} and~\ref{app:visualizations}.

\section{Experiments}
\label{sec:experiments}

\subsection{Datasets and Experiment Setup}
\label{sec:setup}

We evaluate MST on four SEED-family datasets against LaBraM, EEGPT, CBraMod, CSBrain, BIOT, and TFM-Tokenizer. The datasets record 62-channel EEG during emotion-eliciting film clips, with 3, 4, 5, and 7 emotion classes for SEED, SEED-IV, SEED-V, and SEED-VII, respectively \citep{zheng2015investigating,zheng2018emotionmeter,liu2021comparing,jiang2024seedvii}. Appendix~\ref{app:datasets} details subjects, trials, filtering, downsampling, source-derived normalization, and segmentation.

We adopt leave-one-subject-out (LOSO) cross-validation without target-subject calibration or fine-tuning, reporting mean $\pm$ sample SD across subjects after averaging each subject's metrics over three seeds. LaBraM, EEGPT, CBraMod, CSBrain, and BIOT are evaluated with pretrained and randomly initialized weights; TFM-Tokenizer uses its released pretrained tokenizer. Hyperparameters are fixed in advance, and the final training step supplies the test model; no validation split is used. All SEED models receive the same two-second target windows. For models using channel-wise $z$-scores, pooled source-training statistics normalize the held-out subject. LaBraM retains EEG in microvolts and divides it by $100$, following its released input pipeline. MST separately uses unlabeled neighboring windows for its spectral reference. Appendix~\ref{app:datasets} gives training and preprocessing details, and Appendix~\ref{app:paired} reports paired statistics.

MST uses AdamW with learning rate $2\times10^{-4}$, batch size 128, 50,000 steps, and three random seeds. Baselines use fixed model-specific settings listed in Appendix~\ref{app:datasets}. The main comparison evaluates these complete training recipes, including their budgets and augmentation choices.

We additionally use FACED~\citep{chen2023faced} for an additional evaluation of context-aware tokenization, covering nine emotion classes from 123 subjects recorded with a different acquisition setup. Appendix~\ref{app:faced} details its separate evaluation protocol.

\subsection{Main Results}
\label{sec:main-results}

Table~\ref{tab:main} reports the classification accuracy across all four SEED-family datasets. The full Cohen's $\kappa$ and weighted F1 breakdown for every dataset is given in Tables~\ref{tab:full-metrics} and~\ref{tab:full-metrics-cont} in Appendix~\ref{app:full-metrics}.

\begin{table}[!ht]
\centering
\caption{Cross-subject leave-one-subject-out performance on the four SEED-family datasets. All methods are run under the same no-calibration protocol. Accuracy (\%) is averaged over three seeds per subject, then summarized as mean $\pm$ sample SD across subjects. The best performance per column is shown in bold. TFM-Tokenizer uses the authors' pretrained time-frequency tokenizer with a fine-tuned downstream classifier and has no from-scratch variant.}
\label{tab:main}
\small
\begin{tabular}{lcccc} \toprule
    Method & SEED & SEED-IV & SEED-V & SEED-VII \\
    \midrule
    \multicolumn{5}{l}{\textit{Pretrained models / tokenizer with downstream fine-tuning}} \\
    LaBraM             & 62.4 $\pm$ 7.1 & 35.9 $\pm$ 3.6 & 32.5 $\pm$ 4.9 & 24.7 $\pm$ 3.6 \\
    EEGPT                & 56.2 $\pm$ 5.8 & 32.8 $\pm$ 3.0 & 28.6 $\pm$ 3.0 & 23.6 $\pm$ 3.1 \\
    CBraMod            & 57.0 $\pm$ 6.9 & 35.0 $\pm$ 3.6 & 31.5 $\pm$ 5.1 & 23.8 $\pm$ 3.2 \\
    CSBrain            & 61.5 $\pm$ 6.5 & 36.5 $\pm$ 3.8 & 31.0 $\pm$ 4.4 & 23.2 $\pm$ 3.4 \\
    BIOT                  & 59.0 $\pm$ 7.1 & 35.7 $\pm$ 3.7 & 30.6 $\pm$ 4.9 & 23.8 $\pm$ 3.2 \\
    TFM-Tokenizer           & 53.3 $\pm$ 5.6 & 34.0 $\pm$ 3.0 & 28.8 $\pm$ 3.1 & 22.0 $\pm$ 2.1 \\
    \midrule
    \multicolumn{5}{l}{\textit{Same architectures, trained from scratch}} \\
    LaBraM (scratch)                          & 59.8 $\pm$ 6.8 & 30.6 $\pm$ 2.0 & 27.8 $\pm$ 2.4 & 21.2 $\pm$ 2.3 \\
    EEGPT (scratch)                           & 57.6 $\pm$ 7.4 & 32.7 $\pm$ 3.2 & 28.3 $\pm$ 2.7 & 19.6 $\pm$ 2.0 \\
    CBraMod (scratch)                         & 54.5 $\pm$ 6.7 & 31.5 $\pm$ 2.0 & 28.2 $\pm$ 2.9 & 21.7 $\pm$ 2.4 \\
    CSBrain (scratch)                         & 59.7 $\pm$ 6.9 & 34.3 $\pm$ 4.6 & 29.8 $\pm$ 3.5 & 22.0 $\pm$ 3.6 \\
    BIOT (scratch)                            & 59.1 $\pm$ 6.8 & 35.8 $\pm$ 2.8 & 32.4 $\pm$ 4.6 & 24.6 $\pm$ 4.1 \\
    \midrule
    \multicolumn{5}{l}{\textit{Ours (no pretraining)}} \\
    Morlet Spectral Transformer       & \textbf{66.5 $\pm$ 8.1} & \textbf{40.7 $\pm$ 5.1} & \textbf{36.1 $\pm$ 4.5} & \textbf{27.9 $\pm$ 3.9} \\ \bottomrule
\end{tabular}
\end{table}

\paragraph{MST improves SEED accuracy without pretraining.}
MST achieves 66.5\% classification accuracy on SEED under LOSO evaluation. Pretrained LaBraM reaches 62.4\% and CSBrain 61.5\%, while LaBraM trained from scratch reaches 59.8\%. BIOT and TFM-Tokenizer achieve 59.0\% and 53.3\%, respectively. These comparisons motivate studying the representation choices beyond the use of spectral inputs alone. Subject-paired analysis gives a 4.07-point gain over pretrained LaBraM, with a 95\% confidence interval of [1.85, 6.19].

\paragraph{The advantage persists across emotion taxonomies.}
MST maintains its performance advantage on SEED-IV, SEED-V, and SEED-VII, which contain four, five, and seven emotion classes. Against the strongest evaluated pretrained baselines, its paired gains are 4.15, 3.58, and 3.15 percentage points. Against the strongest from-scratch baseline, BIOT on these datasets, the gains are 4.87, 3.69, and 3.26 points. All four dataset comparisons remain significant after Holm correction within each comparator group, with $p_H\leq0.00458$ for pretrained references and $p_H\leq0.000916$ for from-scratch references. Appendix~\ref{app:paired} reports the confidence intervals and statistical protocol.

\paragraph{Paired gains are distributed across subjects.}
\begin{figure}[!ht]
\centering
\includegraphics[width=.78\textwidth]{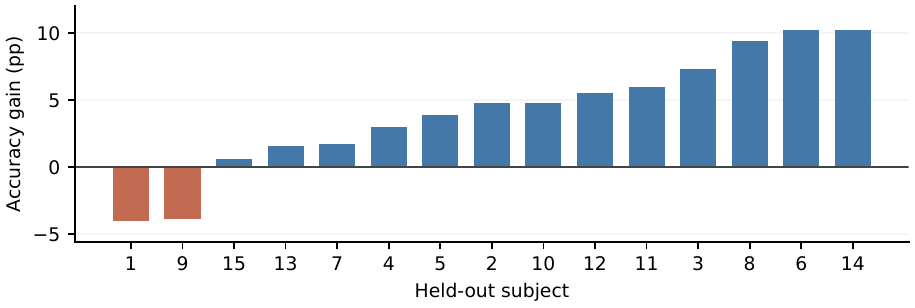}
\caption{Subject-paired LOSO accuracy differences on SEED. Each bar compares MST with pretrained LaBraM on the same held-out subject. Positive values favor MST. Subjects are ordered by their paired difference.}
\label{fig:per-subject}
\end{figure}

Formal paired analyses use a comparator fixed across subjects, with the statistical protocol specified in Appendix~\ref{app:paired}. Figure~\ref{fig:per-subject} shows MST's gains over pretrained LaBraM on 13/15 SEED subjects. Against the fixed strongest from-scratch comparator, MST wins on 13/15, 15/15, 15/16, and 18/20 subjects across the four datasets. Selecting the best pretrained baseline separately for each subject gives 12/15 wins and a 2.07-point mean gain on SEED. Appendix~\ref{app:oracle} reports this subject-wise oracle comparison for all four datasets and both baseline groups.

\begin{table}[!htbp]
\centering
\caption{FACED nine-class comparison on 123 subjects. All models train from scratch on one-second target windows. MST uses five neighbors per side. All models use the checkpoint at their final scheduled training step. Accuracy and weighted F1 are percentages; $\kappa$ is unscaled.}
\label{tab:faced}
\small
\begin{tabular}{lrrr}
\toprule
Model & Accuracy & $\kappa$ & Weighted F1 \\
\midrule
MST (ours) & \textbf{17.52} & \textbf{0.0625} & \textbf{14.33} \\
BIOT & 16.37 & 0.0496 & 13.81 \\
CSBrain & 16.25 & 0.0481 & 14.08 \\
CBraMod & 14.34 & 0.0010 & 3.92 \\
EEGPT & 14.21 & 0.0162 & 8.84 \\
LaBraM & 14.04 & 0.0173 & 10.65 \\
\bottomrule
\end{tabular}
\end{table}

\paragraph{Cross-subject evaluation on FACED.}
We evaluate context-aware tokenization on FACED using six randomly initialized encoders and ten subject-disjoint folds. Each target window lasts one second. MST uses five preceding and five subsequent same-trial windows to construct its local spectral reference. All models are evaluated using the checkpoint at their final scheduled training step, without performance-based checkpoint selection. MST achieves the highest accuracy, $\kappa$, and weighted F1 (Table~\ref{tab:faced}). Its subject-paired accuracy gain over BIOT is 1.15 points, with a 95\% CI of [0.66, 1.64] and 77 wins, 3 ties, and 43 losses. These results extend the evaluation of MST to nine-class emotion decoding under a different acquisition setup. Appendices~\ref{app:faced} and~\ref{app:faced-results} give settings and paired analyses.

\subsection{Ablation Studies}
\label{sec:ablations}

We evaluate the spectral decomposition, reference branch, spatial projection, and augmentations through leave-one-subject-out ablations on all 15 SEED subjects, keeping the split, training budget, and set of three random seeds fixed across configurations. Table~\ref{tab:ablation} summarizes subject accuracies averaged over the three seeds, with subjects weighted equally. Appendix~\ref{app:ablation-detail} provides per-subject accuracy and additional metrics for a four-subject subset.

\begin{table}[!htbp]
\centering
\caption{Component ablations on all 15 SEED subjects. Accuracy is averaged over three seeds per subject, then reported as mean $\pm$ sample SD across subjects. Mean paired drop is full MST minus the ablated model, computed from the displayed means in percentage points.}
\label{tab:ablation}
\small 
\begin{tabular}{lcc}
\toprule
Configuration & Accuracy (\%) & Mean paired drop \\
\midrule
MST (full)                & \textbf{66.5 $\pm$ 8.1} & --- \\
STFT front-end            & 54.7 $\pm$ 7.1 &  11.8 \\
w/o neighbor-reference branch      & 60.9 $\pm$ 8.0 &  5.6 \\
Shared spatial projection & 63.2 $\pm$ 6.8 &  3.3 \\
\midrule
w/o data augmentation     & 60.1 $\pm$ 6.3 &  6.4 \\
w/o phase perturbation    & 63.2 $\pm$ 6.0 &  3.3 \\
w/o band-specific noise   & 62.8 $\pm$ 6.6 &  3.7 \\
w/o channel dropout       & 62.3 $\pm$ 7.6 &  4.2 \\
w/o wavelet time roll     & 61.2 $\pm$ 6.2 &  5.3 \\
\bottomrule
\end{tabular}
\end{table}

\paragraph{Morlet tokenization contributes the largest ablation gain.}
Replacing the Morlet transform with an STFT front-end reduces mean accuracy by 11.8 percentage points. The STFT uses a 400\,ms Hann window and 200\,ms hop, binned into the same 20 frequencies and 16 time frames. The downstream stages remain identical. The 11.8-point difference is the largest among the component ablations in Table~\ref{tab:ablation}.

\paragraph{The neighbor-reference branch improves decoding accuracy.}
Removing the residual branch also removes all neighbor information and reduces mean accuracy by 5.6 points. The ablation measures the joint contribution of neighbor information and its residual/projection pathway. To compare complete models with the same visible context, we also compare MST with a CSBrain control receiving the same centered 22-second span on all 15 SEED subjects. MST reaches 66.5\% versus 60.5\% for context-matched CSBrain, a paired gain of 5.99 points with 95\% CI [3.90, 8.14]. Using ten preceding windows gives 66.1\% accuracy, with a centered-minus-past-only difference of 0.36 points and CI [$-$0.02, 0.72]. Adding context to CSBrain yields 0.83 points, CI [0.00, 1.69], and Holm $p=0.1738$. Each architecture keeps its own training recipe across context variants: 10,000 steps for CSBrain and 50,000 for MST. The matched-span comparison evaluates these complete models, while the past-only result demonstrates prediction using preceding context. Table~\ref{tab:context} summarizes these controls, with settings and paired comparisons in Appendix~\ref{app:context-results}. 

\begin{table}[!htbp]
\centering
\caption{Context controls on all 15 SEED subjects. Accuracy is averaged over three seeds per subject, then reported as mean $\pm$ sample SD across subjects. MST retains 320 patch tokens in both context settings.}
\label{tab:context}
\small
\begin{tabular}{lcc}
\toprule
Model & Context beyond the target & Accuracy (\%) \\
\midrule
CSBrain & None & 59.7 $\pm$ 6.9 \\
CSBrain, context-matched & Five previous + five next windows & 60.5 $\pm$ 7.6 \\
MST, centered & Five previous + five next windows & \textbf{66.5 $\pm$ 8.1} \\
MST, past-only & Ten previous windows & 66.1 $\pm$ 8.2 \\
\bottomrule
\end{tabular}
\end{table}

\paragraph{Frequency-specific spatial projection supports band-dependent weighting.}
Sharing one spatial projection across frequencies reduces mean accuracy by 3.3 points, supporting frequency-dependent electrode mixing. Figure~\ref{fig:spatial_filters} makes these combinations inspectable. The scalp maps visualize electrode mixing in this layer; Appendix~\ref{app:visualizations} describes their interpretation.

\paragraph{The augmentation stack improves robustness across subjects.}
Disabling all augmentations reduces mean accuracy by 6.4 points. Removing phase perturbation, band-specific noise, channel dropout, or wavelet time roll individually gives mean drops of 3.3, 3.7, 4.2, and 5.3 points, respectively. Wavelet time roll contributes the largest gain among the individual augmentations. Every ablated configuration has lower mean accuracy than full MST across the 15 held-out subjects.

\paragraph{Pretraining-free training and computational cost.}
MST has 9.52 million parameters and uses 31.71 GFLOPs per SEED prediction. One LOSO fold takes approximately six hours on an NVIDIA RTX Pro 6000 Blackwell Max-Q. Appendix~\ref{app:compute} compares parameters, FP32 state storage, and forward operations across models.

\section{Conclusions}

MST uses a neighboring spectral summary to condition EEG tokens before attention. Morlet spectra and frequency-specific projections retain both absolute activity and reference-relative deviations, allowing local context to shape a fixed target token grid. Across the SEED-family datasets, MST improves accuracy over the evaluated encoders. Component ablations show the contributions of Morlet tokenization, the neighbor-reference pathway, and frequency-specific mixing.

\paragraph{Future directions.}
The evaluation covers held-out subjects across four SEED-family datasets and the distinct FACED acquisition setup, with models trained separately for each dataset. The past-only reference extends contextual tokenization to prediction from preceding windows. Future work will examine transfer across datasets and stimuli, latency and accuracy during continuous interaction, and integration as an EEG perception component for emotion-sensitive agents.

\clearpage
\subsection*{AI use statement}
Generative AI tools assisted with research code and manuscript writing, including language revision, reference checking, mathematical derivations, statistical analysis, and LaTeX formatting. GPT Image assisted with the schematic architecture illustration in Figure~\ref{fig:architecture}. 

\subsection*{Ethics statement}
This study uses existing EEG datasets and collects no new human recordings. Access and reuse follow the providers' research agreements. EEG and affective labels can reveal sensitive information. Applications require informed consent, privacy protection, and evaluation across populations and recording conditions; benchmark emotion labels depend on the elicitation protocol.

\subsection*{Reproducibility statement}
Section~\ref{sec:method} defines the representation and architecture. Appendix~\ref{app:datasets} describes preprocessing and training, Appendix~\ref{app:de-equivalence} provides the Gaussian derivation, and Appendices~\ref{app:paired}--\ref{app:compute} give supplementary results and analysis settings.

\bibliographystyle{iclr2027_conference}
\bibliography{paper}

\clearpage
\appendix

\section{Method Details and Derivations}
\label{app:method-details}

This appendix specifies the Morlet transform, derives its relation to differential entropy, and describes the local reference and its behavior at trial boundaries.

\subsection{Morlet Wavelet Bank and Discrete Tokenization}
\label{app:morlet-bank}

Figure~\ref{fig:morlet_bank} illustrates the fixed Morlet transform used throughout this article, together with an illustrative example of tokenization of a single EEG channel. The figure comprises three panels. Panel (a) shows the real parts of representative complex Morlet kernels centered at low, middle, and high frequencies on a shared time axis. Panel (b) displays the $F = 20$ log-spaced center frequencies overlaid on the canonical delta, theta, alpha, beta, and gamma bands. Panel (c) shows a sample raw EEG segment and its log-amplitude time-frequency grid $S_{c, f, \tau}$ obtained via complex wavelet convolution followed by adaptive temporal pooling into $N_\tau = 16$ bins.

\begin{figure}[!htbp]
\centering
\includegraphics[width=.68\textwidth]{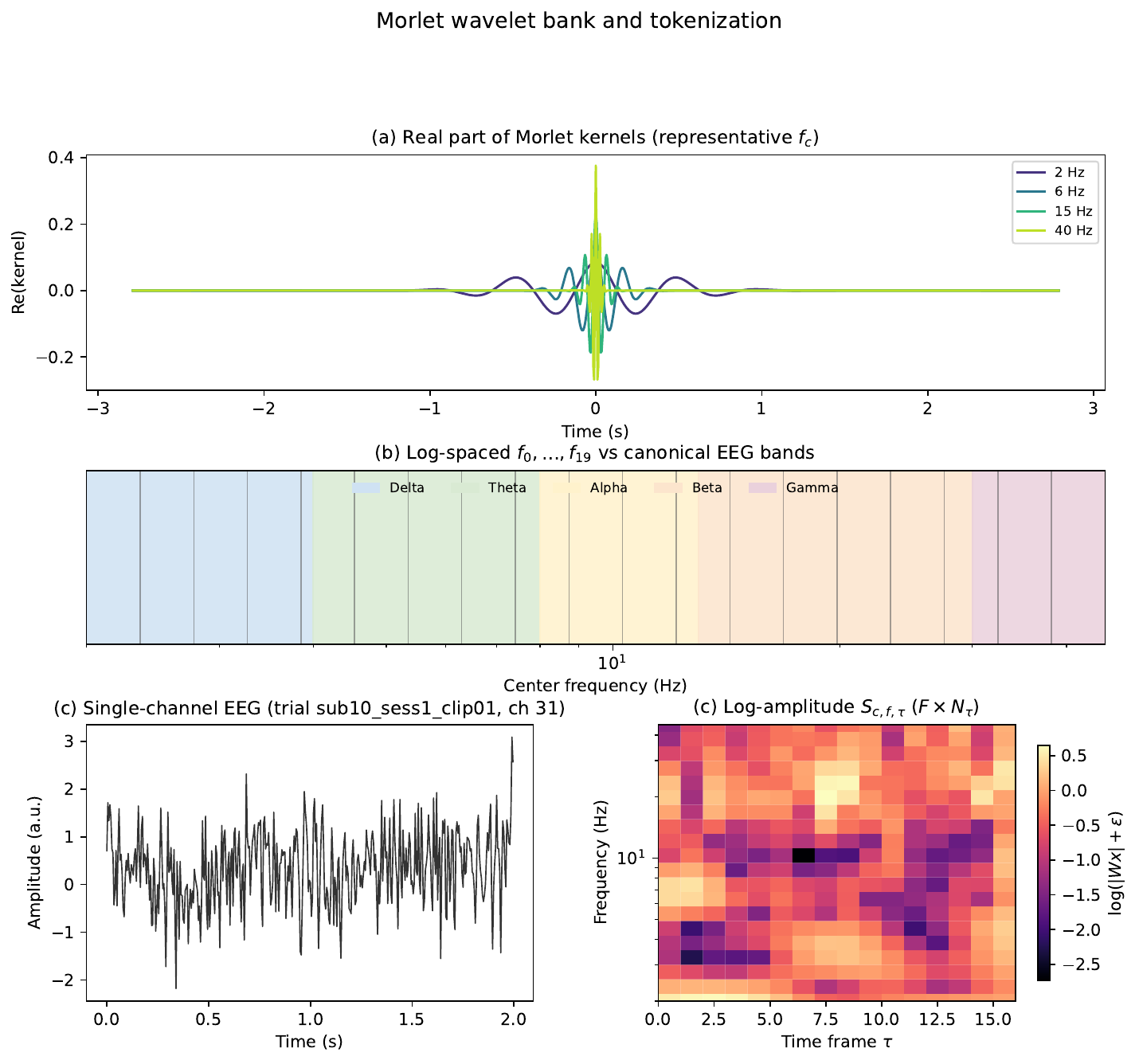}
\caption{Morlet wavelet tokenization. (a) The complex Morlet kernels at different center frequencies, with longer temporal windows at low frequencies and shorter windows at high frequencies. (b) The log-spaced center frequencies aligned with canonical EEG bands. (c) An example raw EEG channel and its corresponding log-amplitude time-frequency grid produced by the wavelet bank and adaptive time pooling.}
\label{fig:morlet_bank}
\end{figure}

The center frequencies are $f_k=2(45/2)^{k/(F-1)}$ for $k=0,\ldots,F-1$. All kernels share the cycle parameter $n_{\mathrm{cyc}}=5$, with Gaussian temporal standard deviation $\sigma_f=n_{\mathrm{cyc}}/(2\pi f)$. Thus $\sigma_f$ is about 0.398\,s at 2\,Hz and 0.0177\,s at 45\,Hz. The number of cycles is constant while the absolute temporal support changes with frequency. Temporal pooling produces 16 bins, retaining local amplitude variation within each segment.

\paragraph{Discrete transform and numerical stabilizers.}
At sampling rate $f_s=200$\,Hz, the kernel for center frequency $f$ has radius $r_f=\lceil7\sigma_f f_s\rceil$ samples and support $-r_f,\ldots,r_f$. Each complex kernel is normalized to unit discrete $\ell_2$ norm, then symmetrically zero-padded to the longest kernel for batched convolution. The input window is zero-padded during convolution to preserve its 400-sample length. At 2\,Hz, $r_f=558$ and the kernel has 1117 samples, so low-frequency coefficients include substantial padding support. For real and imaginary responses $u$ and $v$, the numerical envelope is $\sqrt{u^2+v^2+\delta}$ with $\delta=10^{-12}$. Adaptive average pooling maps the 400 samples to 16 temporal bins, and $\log(A+\epsilon)$ is applied with $\epsilon=10^{-6}$. The population relation below describes the limit of negligible stabilizers.

\subsection{Relation Between Morlet Tokenization and Differential Entropy}
\label{app:de-equivalence}
For a real zero-mean Gaussian bandpass signal $x_b$ with variance $P_b$, the differential entropy is
\begin{equation}
h(x_b)=\tfrac12\log(2\pi eP_b).
\label{eq:gaussian-de}
\end{equation}
This identity motivates the log-variance feature used in EEG emotion recognition \citep{cover2006elements,duan2013differential}. For the complex Morlet response, it is useful to distinguish total complex power from the variance of either quadrature.

\begin{proposition}[Log-envelope statistic under a proper Gaussian response]
\label{prop:morlet-de}
Let $z_f=u_f+jv_f$ be a zero-mean proper complex Gaussian response, with independent real and imaginary parts of variance $q_f>0$. Let $Q_f=\mathbb{E}|z_f|^2=2q_f$. Then
\begin{equation}
\log\mathbb{E}|z_f|
=\tfrac12\log Q_f+\tfrac12\log(\pi/4)
=h(u_f)-\tfrac12\log(4e).
\label{eq:de-relation}
\end{equation}
If the pooled mean envelope converges to $\mathbb{E}|z_f|$ and the amplitude and log stabilizers tend to zero, the corresponding token converges to this quantity.
\end{proposition}
\begin{proof}
Since $u_f$ and $v_f$ are independent $\mathcal{N}(0,q_f)$ variables, $|z_f|$ is Rayleigh with scale $\sqrt{q_f}$ \citep{rice1944mathematical,papoulis2002probability}. Consequently,
\[
\mathbb{E}|z_f|=\sqrt{q_f}\sqrt{\pi/2},
\qquad \mathbb{E}|z_f|^2=2q_f=Q_f.
\]
Taking logarithms yields $\log\mathbb{E}|z_f|=\frac12\log Q_f+\frac12\log(\pi/4)$. Applying Equation~\eqref{eq:gaussian-de} to $u_f$ gives $h(u_f)=\frac12\log(2\pi e q_f)$, so the difference is $-\frac12\log(4e)$. Continuity of the logarithm at the positive mean envelope establishes the convergence statement.
\end{proof}

For a stationary input with two-sided power spectral density $S_{xx}$ and kernel response $\widehat\psi_f$, total complex response power is
\[
Q_f=\int_{-\infty}^{\infty}|\widehat\psi_f(\nu)|^2S_{xx}(\nu)\,d\nu.
\]
The quadrature variance is $q_f=Q_f/2$ under the proper-response assumption. This definition keeps the factor of two consistent across the spectral and Gaussian calculations.

A narrow-band, approximately analytic Morlet response motivates the proper-Gaussian approximation. The finite-window token is an empirical log mean envelope, and its relation to Equation~\eqref{eq:de-relation} depends on stationarity and adequate averaging. In particular, a short temporal bin need not approximate the population mean at low frequencies. The proposition characterizes the population statistic and its averaging limit.

The correspondence is with the entropy of a Morlet-filtered quadrature, whose passband is determined by the kernel. Canonical five-band DE instead integrates power over its designated passbands. Recovering that statistic requires matching the filters and combining band powers before taking the logarithm. MST uses a denser frequency grid and multiple temporal bins to retain localized spectral information.

\subsection{Local Spectral Reference and Boundary Behavior}
\label{app:context}
Let $A_{t,c,f,\tau}$ be the pooled stabilized Morlet envelope of the current window and $\mathcal{N}_t$ the fixed set of ten neighboring slots. Every unavailable slot contains a zero EEG window before the Morlet transform. Its envelope is $A_0=\sqrt{\delta}$, so its log feature is $s_0=\log(\sqrt{\delta}+\epsilon)=\log(2\times10^{-6})$. All ten slots enter the mean. Subtracting their mean log amplitude gives
\begin{equation}
S^{\mathrm{res}}_{t,c,f,\tau}
=\log\frac{A_{t,c,f,\tau}+\epsilon}
{\left(\prod_{j\in\mathcal{N}_t}(A_{j,c,f,\tau}+\epsilon)\right)^{1/|\mathcal{N}_t|}}.
\label{eq:ratio}
\end{equation}
Thus the residual is a log ratio to the geometric mean over the fixed reference slots. For a complete neighborhood and negligible stabilizers, a positive amplitude gain shared by the current and reference windows cancels. At trial boundaries, let $m_t$ be the number of available neighbors. The reference is $(\sum_{j\text{ available}}S_j+(10-m_t)s_0)/10$, rather than a mean renormalized by $m_t$. Boundary residuals therefore also depend on the number of missing slots. For a trial with $n$ target windows, $\min(10,n)$ targets have incomplete context in either the centered or past-only setting. The original branch retains absolute spectral information, including sustained affective activity.

Equation~\eqref{eq:reference-token} shows how the absolute and referenced spectra enter each token through frequency-specific projections.

Centered and past-only context controls, their training settings, and paired results are reported in Appendix~\ref{app:context-results}.

\section{Experimental Protocols and Implementation}
\label{app:datasets}

This appendix collects the SEED-family data pipeline, baseline adaptations, and training settings, followed by the separate FACED protocol.

\subsection{SEED Datasets, Preprocessing, and Evaluation}
\label{app:seed-protocol}
All four datasets in the SEED family~\citep{zheng2015investigating, zheng2018emotionmeter, liu2021comparing, jiang2024seedvii} share the same recording paradigm. Subjects watch emotion-eliciting film clips while 62-channel EEG is recorded with a Neuroscan SynAmps2 system at 1000~Hz. The datasets differ in the number of emotion classes, subjects, clips, and sessions, as summarized in Table~\ref{tab:datasets}.

\begin{table}[!htbp]
  \centering
  \caption{Per-dataset configurations of the SEED family.}
  \label{tab:datasets}
  \small
  \begin{tabular}{lcccc}
    \toprule
    Dataset & Classes & Subjects & Clips per session & Sessions \\
    \midrule
    SEED      & 3 & 15 & 15 & 3 \\
    SEED-IV   & 4 & 15 & 24 & 3 \\
    SEED-V    & 5 & 16 & 15 & 3 \\
    SEED-VII  & 7 & 20 & 20 & 4 \\
    \bottomrule
  \end{tabular}
\end{table}

\paragraph{Preprocessing.}
Raw EEG is band-pass filtered between 0.1 and 75\,Hz, notch-filtered at 50\,Hz, and downsampled to 200\,Hz. We retain the 62 channels in a consistent order, with no ICA or manual artifact rejection. For models other than LaBraM, each source subject is standardized channel-wise using statistics from its training trials within each LOSO fold. Statistics pooled across source training trials normalize the held-out subject. LaBraM uses EEG in microvolts divided by $100$, without this $z$-score normalization. Recordings are segmented into non-overlapping two-second windows of 400 samples, and each window inherits the label of its parent clip. The unlabeled neighboring windows used by the local spectral reference are separate from this channel normalization.

\paragraph{Evaluation protocol.}
Each LOSO fold holds out all recordings from one subject. All remaining subjects supply the training data, with no validation partition. Each model uses hyperparameters fixed before the LOSO evaluation and is tested at its final scheduled training step, without checkpoint selection or early stopping. The held-out subject supplies no labels or statistics for channel normalization or training decisions. The four tasks yield 15, 15, 16, and 20 folds, totaling 66 subject--dataset folds.

\paragraph{Hardware.}
All models, including baselines, are trained on a single NVIDIA RTX Pro 6000 Blackwell Max-Q GPU. One LOSO fold for MST takes around 6 hours. 

\paragraph{Dataset access and licenses.}
The SEED-family datasets are provided by the BCMI laboratory at Shanghai Jiao Tong University and are cited in Section~\ref{sec:setup}. They are used under the SJTU Emotion EEG Dataset License Agreement, which permits academic research use, prohibits commercial use and redistribution, requires approved access through the SEED website, and asks users to cite the relevant dataset papers.

\subsection{Baseline Implementations}
\label{app:baseline-implementations}
\paragraph{Common adaptation protocol.} 
To ensure the protocol consistency across datasets, we re-implement and re-run the baselines, LaBraM, EEGPT, CBraMod, CSBrain, BIOT, and TFM-Tokenizer, under the same strict LOSO protocol described in the experiment setup section, rather than citing numbers from prior benchmarks. Each baseline receives the same 2-second EEG windows with shape $B \times T \times C$, where $T=400$ and $C=62$. LaBraM uses its released microvolt-scaling convention, while the other baselines use source-derived $z$-score normalization. The output dimension of every classifier is set to the number of emotion classes in the target dataset. The pretrained variant loads the authors' released checkpoint and fine-tunes per fold on the source subjects only, using the downstream wrapper and training configuration for that model. The from-scratch variant trains the identical architecture from random initialization on the source subjects only, with no external pretraining data. TFM-Tokenizer is evaluated with the released tokenizer and encoder checkpoints, and therefore has no from-scratch variant in our table.

\paragraph{Model-specific wrappers.}
LaBraM uses our wrapper around the official \texttt{NeuralTransformer} fine-tuning model. The 2-second input is split into two non-overlapping 1-second patches of length $200$ for every channel and reshaped to $B \times C \times 2 \times 200$. Following the released LaBraM fine-tuning pipeline, the input remains in microvolts until it is divided by $100$. No channel-wise $z$-score normalization is applied to LaBraM inputs. We use the base setting with embedding dimension $200$, depth $12$, $10$ attention heads, MLP ratio $4$, output channel count $8$, drop path rate $0.1$, and mean pooling. The 62 SEED channels are mapped to LaBraM 10-20 positional indices through the channel-name lookup copied from the original code. During pretrained runs, \texttt{labram-base.pth} is loaded after removing pretraining-only keys and remapping the normalization weights used by the fine-tuning backbone. The downstream head is a linear layer applied to the mean of all patch tokens.

EEGPT uses the official EEGTransformer encoder. We use a wrapper to transpose each window to $B \times C \times T$ and apply a channel-wise patch embedding with patch size $64$ and stride $64$ over the 400-sample sequence. We use target length $400$, target channels $62$, embedding dimension $512$, four summary tokens, depth $8$, $8$ attention heads, MLP ratio $4$, zero dropout, and no channel projection. The released \path{eegpt_mcae_58chs_4s_large4E.ckpt} checkpoint is loaded by stripping training prefixes, discarding reconstruction and prediction heads, and partially copying compatible channel embeddings when needed. The four summary streams are averaged across time patches and streams, then passed to a linear classifier.

CBraMod uses the official criss-cross Transformer backbone. We use a wrapper to reshape each 400-sample window to $B \times C \times 2 \times 200$. Its patch embedding combines a temporal convolutional projection with an FFT magnitude projection over $101$ one-sided frequency bins. We use $d_{\mathrm{model}}=200$, feed-forward dimension $800$, $12$ layers, $8$ heads, sequence length $30$, patch length $200$, and two patches per channel. The original projection head is replaced by an identity module so that the downstream classifier operates on backbone features. The reported classifier is average pooling over channels and patches followed by a linear layer. For pretrained runs, \texttt{pretrained\_weights.pth} is loaded into the backbone after classifier and output-projection keys are removed.

CSBrain uses the official CSBrain downstream wrapper. It shares the same raw window shape, patch length, depth, hidden dimension, feed-forward dimension, and average-pooling classifier as CBraMod. In addition, it sorts the 62 channels according to the SEED topological order, groups them into the five brain regions used by the original implementation, and applies temporal embedding kernels of size $1$, $3$, and $5$ inside each encoder layer. Its patch embedding again combines temporal convolution and FFT magnitude features before the CSBrain regional attention blocks. The pretrained run loads \texttt{CSBrain.pth} into shape-compatible backbone parameters only, while leaving the downstream classifier randomly initialized.

BIOT uses the official BIOT encoder. We use a wrapper to transpose the input to $B \times C \times T$ and transform each channel by STFT with $n_{\mathrm{fft}}=200$, hop length $100$, and no centering. The magnitude spectrum has 101 bins per frame and three frames per two-second window. A linear spectral embedding, channel embeddings, and sinusoidal temporal positions form $62\times3=186$ tokens of dimension 256. A four-layer, eight-head encoder processes these tokens, followed by mean pooling and a linear classifier. Compatible pretrained encoder tensors are loaded directly, and available channel-embedding rows are copied into the 62-channel table.

TFM-Tokenizer receives each channel's raw 400 samples and its Hann-window STFT, using $n_{\mathrm{fft}}=200$, hop 100, no centering, and frequency indices 0--99. The pretrained vector-quantized tokenizer has 8192 entries of dimension 64 and remains frozen. Its codes are reshaped by channel and processed by the released four-layer, eight-head classifier, with 62 channel slots and a task-specific output head. The encoder checkpoint initializes compatible parameters except the classification head; additional channel embeddings are initialized separately. The downstream classifier is fine-tuned while the tokenizer is fixed.

\subsection{SEED Baseline Training Settings}
\label{app:baseline-training}
Table~\ref{tab:seed-training} reports the fixed training recipes for the SEED comparisons. All baselines use AdamW with $\beta=(0.9,0.999)$, $\epsilon=10^{-8}$, weight decay $0.05$, gradient clipping at norm $1.0$, and three random seeds. Biases and one-dimensional parameters receive no weight decay. The listed learning rate applies to all trainable parameters, with linear warmup followed by cosine decay. Gradient accumulation is one, so the listed batch size is also the effective batch size. Training uses unweighted cross-entropy without label smoothing or input augmentation. The augmentation stack described below is specific to MST.

\begin{table}[!htbp]
\centering
\small
\caption{SEED-family training settings. P and S denote pretrained and scratch initialization. Each model is evaluated at its final listed step. Grouped regimes share all tabulated settings. Budgets are model-specific.}
\label{tab:seed-training}
\begin{tabular}{@{}lllrrrr@{}}
\toprule
Model & Dataset & Init. & Learning rate & Batch & Steps & Warmup \\
\midrule
LaBraM & SEED & P/S & $2\times10^{-4}$ & 128 & 50,000 & 500 \\
 & IV, V & P/S & $2\times10^{-4}$ & 100 & 10,000 & 500 \\
 & VII & P/S & $1.5\times10^{-4}$ & 64 & 10,000 & 500 \\
EEGPT & All four & P/S & $5\times10^{-5}$ & 80 & 10,000 & 1,000 \\
CBraMod & All four & P & $2\times10^{-4}$ & 128 & 10,000 & 1,000 \\
 & All four & S & $2\times10^{-4}$ & 100 & 10,000 & 1,000 \\
CSBrain & All four & P/S & $2\times10^{-4}$ & 100 & 10,000 & 1,000 \\
BIOT & SEED & P/S & $2\times10^{-4}$ & 80 & 50,000 & 1,000 \\
 & IV & P/S & $5.3\times10^{-5}$ & 100 & 50,000 & 1,000 \\
 & V, VII & P/S & $2\times10^{-4}$ & 100 & 50,000 & 1,000 \\
TFM-Tokenizer & All four & P & $2\times10^{-4}$ & 80 & 50,000 & 1,000 \\
\bottomrule
\end{tabular}
\end{table}

Pretrained and scratch runs share these settings except for the CBraMod batch size shown in the table. All backbone parameters are updated for the first five models. TFM-Tokenizer keeps its tokenizer frozen and updates its downstream encoder and classification head. The supplementary \path{supplement/seed_training_manifest.json} records the confirmed configurations, source-file hashes, precision variants, and final-step evaluation rule. The component ablations retain the MST recipe while changing the indicated component.

\subsection{MST Configuration, Training, and Augmentations}
\label{app:mst-training}
For the subject-independent SEED experiments, we train MST with cross-entropy loss for $50{,}000$ steps and evaluate the final-step model. We use AdamW with learning rate $2\times 10^{-4}$, weight decay $0.05$, cosine learning-rate decay, $1000$ warmup steps, batch size $128$, gradient accumulation of $1$, gradient clipping at norm $1.0$, and no mixed precision. We repeat training with three random seeds and report the average performance.

The model configuration follows those values in Section~\ref{sec:setup}. We use $20$ logarithmically spaced Morlet frequencies from $2$ to $45$\,Hz, $5$ wavelet cycles, sampling rate $200$\,Hz, $16$ pooled time frames, log stabilizer $10^{-6}$, $62$ EEG channels, spatial dimension $16$, embedding dimension $256$, $12$ Transformer layers, $8$ attention heads, dropout $0.05$, MLP expansion ratio $4$, and long-context baseline size $K=5$. The supervised contrastive and VICReg auxiliary losses are disabled in the reported runs, with weights set to zero.

\paragraph{Position encoding.}
Each attention head has dimension $256/8=32$. The first 16 dimensions of its queries and keys receive time-axis RoPE, and the remaining 16 receive frequency-axis RoPE. Both axes use base $10{,}000$ with inverse frequencies $10{,}000^{-2j/16}$ for $j=0,\ldots,7$. Positions are the integer time-bin and frequency-bin indices. CLS receives no rotary transformation.

\paragraph{Augmentation probabilities and strengths.}
For data augmentation, phase perturbation is applied with probability $1.0$, band-specific noise with probability $0.5$, channel dropout with probability $0.5$, and post-wavelet time roll with probability $0.5$. Phase perturbation samples its standard deviation uniformly from $\pi/6$ to $\pi/2$. Band-specific noise samples one to three bands and an amplitude scale uniformly from $0.1$ to $0.3$. Channel dropout samples the drop fraction uniformly from $0.1$ to $0.3$. Channel-gain, spectral-tilt, and intra-label interpolation augmentations are disabled.

\paragraph{Augmentation operators.}
Phase noise is Gaussian and sampled per example and Fourier bin, with the same perturbation shared across electrodes. An inverse real FFT reconstructs the window. Band noise partitions 2--45\,Hz into one to three logarithmically spaced intervals. Independent Gaussian noise is filtered to each interval for each channel, normalized to that channel's band RMS, and multiplied by the sampled amplitude scale. Channel dropout draws an independent Bernoulli mask for each example and electrode at the sampled drop probability. The same mask applies to both spectral branches, replacing the dropped entries with one learned value per frequency shared across time and electrodes. Time roll shifts both branches together by a per-example offset drawn uniformly from $1,\ldots,15$. It acts after reference subtraction and before spatial projection.

This numerical configuration is fixed across LOSO folds; hyperparameters are not retuned for individual held-out subjects.

\subsection{FACED Data, Evaluation, and Training Settings}
\label{app:faced}
FACED provides nine emotion classes from 123 subjects \citep{chen2023faced}. This comparison uses MST with local spectral referencing enabled ($K=5$). Inputs use 30 scalp channels at 200\,Hz, non-overlapping one-second target windows, and per-channel $z$-scores computed within each window. The reference averages log-amplitude spectra from five preceding and five subsequent windows within the same trial, providing ten seconds of neighboring context for a complete neighborhood. The original and residual spectra are concatenated before spatial projection, with missing neighbors handled as described in Section~\ref{sec:baseline}. The Morlet representation has 20 frequencies from 4 to 45\,Hz, eight time bins, spatial dimension 16, and a 12-layer encoder of dimension 256. All six models are trained from random initialization with three random seeds, and reported performance is averaged over the three runs.

Subjects are divided into ten contiguous folds in numerical order. The first nine folds contain 12 subjects each and the final fold contains 15. 
Within each fold, every model is evaluated at its final scheduled training step, with model-specific budgets listed in Table~\ref{tab:faced-training}. Held-out performance is not used to select checkpoints. Metrics in Table~\ref{tab:faced} are computed after concatenating the held-out predictions from these final-step checkpoints. Every subject contributes 840 windows, so pooled accuracy equals subject-macro accuracy; pooled $\kappa$ and F1 are computed directly from the combined predictions.

\paragraph{Training settings.}
All models use AdamW, smoothed cross-entropy, gradient clipping at 1.0, and bfloat16 arithmetic. Table~\ref{tab:faced-training} reports model-specific settings. MST uses a constant learning rate after warmup, and the baselines use cosine decay.

\begin{table}[!htbp]
\centering
\caption{FACED training settings. Evaluation interval and warmup are in optimization steps. Reporting uses the final-step checkpoint, without checkpoint selection.}
\label{tab:faced-training}
\small
\begin{tabular}{lrrrrrr}
\toprule
Model & Batch & Steps & Learning rate & Weight decay & Warmup & Eval. interval \\
\midrule
MST & 100 & 15,000 & $2\times10^{-4}$ & 0.10 & 1,000 & 500 \\
BIOT & 32 & 30,000 & $10^{-4}$ & 0.05 & 2,500 & 2,500 \\
CSBrain & 32 & 30,000 & $10^{-4}$ & 0.05 & 2,500 & 2,500 \\
CBraMod & 64 & 15,000 & $5\times10^{-5}$ & 0.10 & 1,500 & 2,500 \\
EEGPT & 80 & 12,000 & $10^{-4}$ & 0.05 & 1,000 & 1,000 \\
LaBraM & 100 & 8,800 & $10^{-4}$ & 0.10 & 800 & 1,000 \\
\bottomrule
\end{tabular}
\end{table}

Subject-paired comparisons of the final-step predictions are reported in Appendix~\ref{app:faced-results}.

\section{Supplementary Results and Statistical Analyses}
\label{app:supplementary-results}

The following sections provide complete metrics and paired analyses for the SEED-family comparisons, context controls, ablation details, the FACED comparison, and computational profiles.

\subsection{Full Evaluation Metric Breakdown}
\label{app:full-metrics}

Tables~\ref{tab:full-metrics} and~\ref{tab:full-metrics-cont} report Cohen's $\kappa$ and weighted F1 for all methods across the four SEED-family datasets, with all values expressed as percentages. For each metric, we first average the three seed results within each held-out subject, then report the equally weighted mean $\pm$ sample standard deviation across subjects (15 for SEED, 15 for SEED-IV, 16 for SEED-V, and 20 for SEED-VII). Thus the SD describes between-subject dispersion after seed averaging. For the SEED dataset, on Cohen's $\kappa$, which accounts for class imbalance, the proposed MST shows a 6.1-percentage-point improvement over pretrained LaBraM and a 10.0-percentage-point improvement over LaBraM trained from scratch. On weighted F1, MST shows a 4.9-percentage-point improvement over pretrained LaBraM and a 6.6-percentage-point improvement over LaBraM trained from scratch. 

\begin{table}[!htbp]
\centering
\caption{SEED and SEED-IV metrics. Both $100\kappa$ and weighted F1 are percentages.}
\label{tab:full-metrics}
\small
\begin{tabular}{lcccc}
\toprule
 & \multicolumn{2}{c}{SEED} & \multicolumn{2}{c}{SEED-IV} \\
\cmidrule(lr){2-3}\cmidrule(lr){4-5}
Model & $100\kappa$ & F1 & $100\kappa$ & F1 \\
\midrule
\multicolumn{5}{l}{\textit{Pretrained}} \\
LaBraM & 43.5 $\pm$ 10.8 & 60.8 $\pm$ 7.8 & 13.7 $\pm$ 4.5 & 33.7 $\pm$ 4.3 \\
EEGPT & 34.2 $\pm$ 8.7 & 54.6 $\pm$ 5.9 & 9.3 $\pm$ 4.5 & 29.6 $\pm$ 4.1 \\
CBraMod & 37.2 $\pm$ 9.6 & 54.5 $\pm$ 8.2 & 12.4 $\pm$ 5.1 & 31.2 $\pm$ 5.5 \\
CSBrain & 42.2 $\pm$ 9.7 & 60.3 $\pm$ 7.0 & 14.3 $\pm$ 5.7 & 33.3 $\pm$ 4.1 \\
BIOT & 38.6 $\pm$ 10.6 & 57.8 $\pm$ 7.4 & 13.5 $\pm$ 4.9 & 33.9 $\pm$ 4.4 \\
TFM-Tokenizer & 29.9 $\pm$ 8.5 & 51.7 $\pm$ 6.0 & 11.2 $\pm$ 4.4 & 31.7 $\pm$ 3.9 \\
\midrule
\multicolumn{5}{l}{\textit{From scratch}} \\
LaBraM (scratch) & 39.6 $\pm$ 10.1 & 59.1 $\pm$ 6.3 & 5.7 $\pm$ 3.3 & 25.8 $\pm$ 3.3 \\
EEGPT (scratch) & 36.4 $\pm$ 11.1 & 56.0 $\pm$ 8.0 & 8.8 $\pm$ 4.4 & 29.2 $\pm$ 4.2 \\
CBraMod (scratch) & 31.6 $\pm$ 10.1 & 53.1 $\pm$ 7.3 & 7.1 $\pm$ 3.5 & 28.4 $\pm$ 3.7 \\
CSBrain (scratch) & 39.5 $\pm$ 10.3 & 58.3 $\pm$ 6.5 & 11.3 $\pm$ 7.0 & 30.2 $\pm$ 5.6 \\
BIOT (scratch) & 38.6 $\pm$ 10.2 & 57.7 $\pm$ 7.4 & 13.0 $\pm$ 3.9 & 33.2 $\pm$ 3.9 \\
\midrule
\multicolumn{5}{l}{\textit{Ours, no pretraining}} \\
MST & \textbf{49.6 $\pm$ 12.2} & \textbf{65.7 $\pm$ 8.7} & \textbf{19.7 $\pm$ 7.2} & \textbf{37.0 $\pm$ 6.7} \\
\bottomrule
\end{tabular}
\end{table}

\begin{table}[!htbp]
\centering
\caption{SEED-V and SEED-VII metrics. Both $100\kappa$ and weighted F1 are percentages.}
\label{tab:full-metrics-cont}
\small
\begin{tabular}{lcccc}
\toprule
 & \multicolumn{2}{c}{SEED-V} & \multicolumn{2}{c}{SEED-VII} \\
\cmidrule(lr){2-3}\cmidrule(lr){4-5}
Model & $100\kappa$ & F1 & $100\kappa$ & F1 \\
\midrule
\multicolumn{5}{l}{\textit{Pretrained}} \\
LaBraM & 13.8 $\pm$ 6.5 & 27.6 $\pm$ 6.7 & 11.0 $\pm$ 4.1 & 21.8 $\pm$ 4.0 \\
EEGPT & 8.0 $\pm$ 4.3 & 23.8 $\pm$ 5.2 & 9.0 $\pm$ 3.6 & 19.8 $\pm$ 3.5 \\
CBraMod & 12.1 $\pm$ 6.8 & 26.4 $\pm$ 6.8 & 9.5 $\pm$ 3.8 & 19.1 $\pm$ 4.1 \\
CSBrain & 10.3 $\pm$ 6.5 & 25.1 $\pm$ 5.8 & \textemdash & 18.1 $\pm$ 4.8 \\
BIOT & 12.2 $\pm$ 6.4 & 25.6 $\pm$ 6.5 & 10.2 $\pm$ 3.9 & 20.1 $\pm$ 3.7 \\
TFM-Tokenizer & \textemdash & 22.4 $\pm$ 5.6 & 6.9 $\pm$ 2.6 & 16.6 $\pm$ 2.6 \\
\midrule
\multicolumn{5}{l}{\textit{From scratch}} \\
LaBraM (scratch) & 5.7 $\pm$ 4.4 & 20.7 $\pm$ 6.5 & 5.6 $\pm$ 3.2 & 15.6 $\pm$ 3.3 \\
EEGPT (scratch) & 5.8 $\pm$ 4.7 & 19.6 $\pm$ 6.0 & 3.4 $\pm$ 2.5 & 11.9 $\pm$ 2.7 \\
CBraMod (scratch) & 7.5 $\pm$ 4.3 & 23.8 $\pm$ 4.8 & 7.0 $\pm$ 3.2 & 17.7 $\pm$ 3.3 \\
CSBrain (scratch) & 9.4 $\pm$ 5.0 & 24.0 $\pm$ 5.2 & 8.0 $\pm$ 4.5 & 18.4 $\pm$ 4.4 \\
BIOT (scratch) & 13.6 $\pm$ 6.7 & 27.1 $\pm$ 7.2 & 10.5 $\pm$ 4.8 & 20.1 $\pm$ 4.8 \\
\midrule
\multicolumn{5}{l}{\textit{Ours, no pretraining}} \\
MST & \textbf{18.3 $\pm$ 5.6} & \textbf{31.1 $\pm$ 5.1} & \textbf{14.7 $\pm$ 4.7} & \textbf{23.7 $\pm$ 4.5} \\
\bottomrule
\end{tabular}
\end{table}

A dash indicates an unreported $\kappa$ summary. Accuracy and weighted F1 are reported for every method and dataset.

\subsection{Subject-Paired Statistical Analysis}
\label{app:paired}
We compare accuracies on the same held-out subjects. For each SEED dataset, the comparator is the model with the highest mean accuracy in the corresponding pretrained or from-scratch group of Table~\ref{tab:main}. Each comparator is fixed across subjects and was frozen from the original main-results table before the paired reanalysis. This is a post hoc comparison against the best reported mean within each group. We report the equally weighted mean difference, subject win count, and a paired-subject percentile bootstrap 95\% confidence interval. The bootstrap uses 200,000 resamples, with \texttt{default\_rng(42)} initialized once and comparisons processed in the table order. Table~\ref{tab:paired} summarizes these analyses using subject accuracies stored to three decimal places.

Two-sided exact paired sign-flip tests assess the mean difference. The four from-scratch comparisons form one Holm family, and the four pretrained comparisons form another. A subject win denotes a higher accuracy for MST, while the test addresses the paired differences across subjects. Confidence intervals describe subject variation conditional on the reported training seed and do not estimate variation across independent retrainings.

\begin{table}[!htbp]
\centering
\caption{Subject-paired SEED accuracy gains. CIs and gains are in percentage points. Each block has a separate four-comparison Holm correction.}
\label{tab:paired}
\small
\begin{tabular}{llccc}
\toprule
Dataset & Reference & Gain [95\% CI] & Wins & Holm $p$ \\
\midrule
\multicolumn{5}{l}{\textit{From-scratch references}} \\
SEED & LaBraM & 6.69 [4.15, 9.26] & 13/15 & 0.000916 \\
SEED-IV & BIOT & 4.87 [3.18, 6.78] & 15/15 & 0.000244 \\
SEED-V & BIOT & 3.69 [1.94, 5.71] & 15/16 & 0.000916 \\
SEED-VII & BIOT & 3.26 [1.78, 4.77] & 18/20 & 0.000916 \\
\midrule
\multicolumn{5}{l}{\textit{Pretrained references}} \\
SEED & LaBraM & 4.07 [1.85, 6.19] & 13/15 & 0.00458 \\
SEED-IV & CSBrain & 4.15 [2.35, 6.11] & 13/15 & 0.00244 \\
SEED-V & LaBraM & 3.58 [1.78, 5.65] & 13/16 & 0.00275 \\
SEED-VII & LaBraM & 3.15 [1.64, 4.70] & 17/20 & 0.00252 \\
\bottomrule
\end{tabular}
\end{table}

\subsection{Comparison with Each Subject's Best Baseline}
\label{app:oracle}
We also compare MST with the highest baseline accuracy on each individual subject. For subject $s$ and baseline group $\mathcal{B}_r$, define
\[
b_s^{(r)}=\max_{m\in\mathcal{B}_r}\mathrm{Acc}_{m,s},
\qquad
\Delta_s^{(r)}=\mathrm{Acc}_{\mathrm{MST},s}-b_s^{(r)}.
\]
The from-scratch group contains LaBraM, EEGPT, CBraMod, CSBrain, and BIOT. The pretrained group contains their pretrained variants and TFM-Tokenizer. Both groups exclude MST. The maximizing model can differ across subjects, forming a subject-wise oracle envelope of the evaluated baselines.

Table~\ref{tab:oracle} reports the fraction of subjects with $\Delta_s^{(r)}>0$ and the mean and median differences, with subjects weighted equally. MST has positive mean and median gains in all eight dataset--group comparisons. On SEED, it exceeds each subject's best pretrained baseline on 12 of 15 subjects, with a mean gain of 2.07 percentage points. This oracle uses held-out outcomes to select a comparator separately for each subject, so these summaries describe performance against that envelope. The formal paired tests in Appendix~\ref{app:paired} use a fixed comparator across subjects.

\begin{table}[!htbp]
\centering
\caption{MST versus each subject's best baseline. A win denotes higher accuracy than every baseline in the specified group on that subject. Mean and median gains are in percentage points.}
\label{tab:oracle}
\small
\begin{tabular}{lrrrr}
\toprule
Dataset & Wins / subjects & Win rate (\%) & Mean gain & Median gain \\
\midrule
\multicolumn{5}{l}{\textit{From-scratch baseline envelope}} \\
SEED & 12/15 & 80.0 & +3.47 & +2.50 \\
SEED-IV & 13/15 & 86.7 & +3.86 & +2.80 \\
SEED-V & 14/16 & 87.5 & +2.73 & +1.55 \\
SEED-VII & 15/20 & 75.0 & +2.29 & +1.95 \\
\midrule
\multicolumn{5}{l}{\textit{Pretrained baseline envelope}} \\
SEED & 12/15 & 80.0 & +2.07 & +3.00 \\
SEED-IV & 11/15 & 73.3 & +2.20 & +1.50 \\
SEED-V & 11/16 & 68.8 & +2.11 & +1.60 \\
SEED-VII & 15/20 & 75.0 & +1.69 & +1.75 \\
\bottomrule
\end{tabular}
\end{table}

\subsection{Temporal Context Controls and Paired Comparisons}
\label{app:context-results}
The reference construction and boundary handling are defined in Appendix~\ref{app:context}.

The centered model uses five preceding and five following two-second windows, providing 20 seconds of reference context and a 22-second visible span including the target. The past-only model uses ten preceding windows. Both retain all target windows, zero-pad missing context at trial boundaries, and keep context within each trial. The past-only reference predicts after observing the current two-second window and uses up to 20 seconds of past EEG within known trial boundaries.

Table~\ref{tab:context-paired} reports the paired context comparisons. Context-matched CSBrain concatenates the previous five, current, and next five windows in temporal order into a 22-second input and predicts the current window. The comparisons in Table~\ref{tab:context-paired} train from scratch using the same 15 SEED LOSO subjects, target windows, normalization, and final-step evaluation rule. The two CSBrain variants share AdamW, learning rate $2\times10^{-4}$, batch size $100$, $10{,}000$ steps, and the same set of three random seeds. The two MST variants share the $50{,}000$-step MST recipe and the same set of three random seeds. Context is matched in visible duration across architectures, with training settings fixed within each architecture. The MST Transformer processes 320 current-window tokens and a CLS token; neighboring windows enter through the baseline computation.

\begin{table}[!htbp]
\centering
\caption{Paired context comparisons. Gains and CIs are in percentage points. The first three comparisons form a Holm family; the final comparison reports its raw $p$ value.}
\label{tab:context-paired}
\small
\begin{tabular}{lccc}
\toprule
Comparison (first minus second) & Gain [95\% CI] & Wins/ties/losses & $p_H$ \\
\midrule
Context CSBrain vs. original & 0.83 [0.00, 1.69] & 11/3/1 & 0.1738 \\
Centered MST vs. context CSBrain & 5.99 [3.90, 8.14] & 14/0/1 & 0.000549 \\
Centered MST vs. past-only MST & 0.36 [$-$0.02, 0.72] & 11/0/4 & 0.1738 \\
Past-only MST vs. context CSBrain & 5.63 [3.35, 8.01] & 13/0/2 & 0.000366$^*$ \\
\bottomrule
\end{tabular}
\end{table}

Context CIs use 200,000 paired-subject percentile bootstrap resamples, and raw $p$ values use exact two-sided sign-flip tests with ties retained. The three primary comparisons jointly assess added input context, MST under matched context, and future-window sensitivity. The supplementary past-only versus context-CSBrain comparison is marked with an asterisk.

\subsection{Four-Subject Ablation Details}
\label{app:ablation-detail}
Tables~\ref{tab:ablation-subjects} and~\ref{tab:ablation-additional} provide supplementary details for a subset of four subjects (9, 5, 6, and 11), while Table~\ref{tab:ablation} summarizes all 15 held-out subjects. The four subjects receive equal weight in the subset summaries. Standard deviations describe between-subject variability and use the sample definition with denominator $n-1$. All eight ablated configurations have lower accuracy than full MST on each of these four subjects.

\begin{table}[!htbp]
\centering
\caption{Per-subject accuracy (\%) averaged over three seeds for a four-subject subset of the SEED ablations.}
\label{tab:ablation-subjects}
\small
\begin{tabular}{lrrrr}
\toprule
Configuration & Subject 9 & Subject 5 & Subject 6 & Subject 11 \\
\midrule
MST (full) & 64.6 & 61.1 & 68.8 & 77.9 \\
STFT front-end & 51.0 & 53.4 & 59.7 & 62.5 \\
w/o neighbor-reference branch & 61.1 & 58.2 & 62.0 & 69.8 \\
Shared spatial projection & 62.3 & 60.1 & 63.8 & 74.7 \\
\midrule
w/o data augmentation & 62.2 & 58.5 & 59.9 & 68.3 \\
w/o phase perturbation & 63.0 & 59.7 & 64.3 & 73.7 \\
w/o band-specific noise & 60.8 & 60.3 & 64.2 & 74.4 \\
w/o channel dropout & 62.9 & 59.4 & 62.9 & 72.5 \\
w/o wavelet time roll & 60.5 & 58.3 & 61.5 & 72.1 \\
\bottomrule
\end{tabular}
\end{table}

\begin{table}[!htbp]
\centering
\caption{Additional ablation metrics, mean $\pm$ sample SD across the same four subjects after averaging each subject's metric over three seeds. Both metrics are expressed on a percentage scale.}
\label{tab:ablation-additional}
\small
\begin{tabular}{lcc}
\toprule
Configuration & $100\kappa$ & Weighted F1 \\
\midrule
MST (full) & 52.0 $\pm$ 11.0 & 66.6 $\pm$ 9.3 \\
STFT front-end & 34.8 $\pm$ 8.4 & 55.2 $\pm$ 5.0 \\
w/o neighbor-reference branch & 43.6 $\pm$ 8.5 & 60.4 $\pm$ 3.6 \\
Shared spatial projection & 47.7 $\pm$ 8.4 & 63.2 $\pm$ 7.5 \\
w/o data augmentation & 44.1 $\pm$ 6.5 & 60.5 $\pm$ 4.3 \\
w/o phase perturbation & 47.8 $\pm$ 7.6 & 63.4 $\pm$ 5.6 \\
w/o band-specific noise & 47.2 $\pm$ 10.1 & 62.4 $\pm$ 7.9 \\
w/o channel dropout & 46.9 $\pm$ 8.3 & 61.4 $\pm$ 7.0 \\
w/o wavelet time roll & 44.2 $\pm$ 8.7 & 61.5 $\pm$ 5.9 \\
\bottomrule
\end{tabular}
\end{table}

\subsection{FACED Subject-Paired Results}
\label{app:faced-results}
The FACED evaluation and final-step checkpoint protocol are specified in Appendix~\ref{app:faced}.

For each subject, we compare MST with each baseline using the held-out predictions from the final-step checkpoints. Table~\ref{tab:faced-paired} reports equal-subject-weight differences and standardized effects $d_z=\bar d/s_d$. Confidence intervals use paired-subject percentile bootstrap with 200,000 resamples and \texttt{default\_rng(42)} initialized once, in the order shown. Two-sided exact sign-flip probabilities are computed by integer subset-sum dynamic programming on the differences in correctly classified windows, with zero differences factored out. The five comparisons form one Holm family. These intervals quantify variation across subjects conditional on the reported training seed and final-step checkpoints.

\begin{table}[!htbp]
\centering
\caption{FACED subject-paired differences for MST. Gains and CIs are in percentage points.}
\label{tab:faced-paired}
\small
\begin{tabular}{lrrrr}
\toprule
Reference & Gain [95\% CI] & Wins/ties/losses & $d_z$ & $p_H$ \\
\midrule
BIOT & 1.15 [0.66, 1.64] & 77/3/43 & 0.414 & $1.08\times10^{-5}$ \\
CSBrain & 1.27 [0.88, 1.67] & 86/2/35 & 0.563 & $1.13\times10^{-8}$ \\
CBraMod & 3.18 [2.58, 3.77] & 99/3/21 & 0.941 & $1.28\times10^{-17}$ \\
EEGPT & 3.31 [2.74, 3.88] & 103/1/19 & 1.026 & $7.31\times10^{-20}$ \\
LaBraM & 3.48 [2.88, 4.08] & 109/1/13 & 1.018 & $1.51\times10^{-19}$ \\
\bottomrule
\end{tabular}
\end{table}

\subsection{Model Size and Forward Computation}
\label{app:compute}
Table~\ref{tab:compute} reports model size and forward-operation counts for the SEED configuration. MST's pretraining-free setting removes an external EEG pretraining stage. The state-storage column measures FP32 model state, not peak runtime memory.

\begin{table}[!htbp]
\centering
\caption{Computational profiles. Parameter counts are in millions and forward operations per prediction are in GFLOPs.}
\label{tab:compute}
\small
\begin{tabular}{lrrrr}
\toprule
Model & Total params & Trainable params & State (MiB) & GFLOPs \\
\midrule
MST & 9.523 & 9.523 & 36.505 & 31.706 \\
LaBraM & 5.821 & 5.821 & 22.204 & 1.593 \\
EEGPT & 25.289 & 25.289 & 96.469 & 20.384 \\
CBraMod & 4.884 & 4.884 & 18.633 & 1.250 \\
CSBrain & 8.862 & 8.862 & 33.807 & 2.716 \\
BIOT & 3.199 & 3.199 & 13.179 & 1.205 \\
\bottomrule
\end{tabular}
\end{table}

\section{Representation Diagnostics and Visualizations}
\label{app:diagnostics}

This appendix documents the sampling and plotting procedures for the spectral diagnostics, learned spatial weights, and augmentation illustrations.

\subsection{Spectral Similarity and Cross-Subject Dispersion}
\label{app:spectral-diagnostic}
Figure~\ref{fig:redundancy} is a descriptive analysis of the spectral representation. The similarity panels use the first sufficiently long trial in the SEED dataset's lexicographic order and 48 consecutive two-second windows with complete neighborhoods. Both heatmaps share the cosine-similarity range $[-1,1]$. The variance panel uses the first 256 windows per subject in lexicographic trial order in each dataset, with the same windows before and after referencing. For this diagnostic, each channel is normalized using that subject's pooled mean and standard deviation. This differs from the source-derived normalization used for classification. We retain the original illustration's $K=10$ neighbors on each side, while the classification experiments use $K=5$. At trial boundaries, unavailable neighbors are zero EEG windows before the Morlet transform and the reference denominator stays at 20.

For subject $s$, the frequency vector $u_{s,f}$ is the mean log-amplitude over the selected windows, channels, and time frames. We compute the corresponding residual vector using the same averaging. For either representation, the displayed statistic is
\begin{equation*}
V=\frac{1}{F}\sum_{f=1}^{F}\frac{1}{|\mathcal S|}\sum_{s\in\mathcal S}
\left(u_{s,f}-\frac{1}{|\mathcal S|}\sum_{r\in\mathcal S}u_{r,f}\right)^2.
\end{equation*}
Subjects have equal weight and variance uses the population denominator. This measures dispersion of subject-mean spectra, rather than variance over individual EEG windows.

\paragraph{SEED-IV scale sensitivity.}
The descriptive SEED-IV panel excludes Subjects 1 and 5, whose pooled channel standard deviations are inflated by high-amplitude recording artifacts. We flag a subject when the pooled standard deviation exceeds the median trial standard deviation by more than tenfold on a majority of channels. The threshold is exceeded on 62 channels for Subject 1 and 60 for Subject 5. Their median channel-wise inflation factors are 34.9 and 17.7, respectively. The same retained subject set is used for both bars. Table~\ref{tab:spectral-diagnostic} also reports the full-cohort result. All 15 SEED-IV subjects remain in the classification evaluation.

\begin{table}[!htbp]
\centering
\caption{Corrected cross-subject spectral dispersion. Both representations use the same 256 windows per subject and $K=10$ neighbors on each side. Figure~\ref{fig:redundancy} uses the 13-subject SEED-IV diagnostic; the 15-subject row gives its full-cohort counterpart.}
\label{tab:spectral-diagnostic}
\small
\begin{tabular}{lrrr}
\toprule
Dataset and cohort & Subjects & Original & Residual \\
\midrule
SEED, all subjects & 15 & 0.095372 & 0.000274 \\
SEED-IV, all subjects & 15 & 1.133925 & 0.015932 \\
SEED-IV, without Subjects 1 and 5 & 13 & 0.059360 & 0.000833 \\
SEED-V, all subjects & 16 & 0.197190 & 0.001083 \\
SEED-VII, all subjects & 20 & 0.177487 & 0.001033 \\
\bottomrule
\end{tabular}
\end{table}

\subsection{Spatial Weights and Augmentation Illustrations}
\label{app:visualizations}
Figure~\ref{fig:spatial_filters} displays the first $C=62$ rows of $W^{(f)}$, corresponding to the original log-spectrum branch. The residual-branch rows are not shown. Rows display spatial components $d=1,\ldots,4$ in index order, and columns use the grid frequencies nearest 3, 6, 10, 20, and 38\,Hz. A shared symmetric color range uses the 99th percentile of absolute displayed weights. Standard 10--20 positions are used where available, with CB1 and CB2 approximated by small offsets from O1 and O2 for plotting. These maps show projection parameters, which subsequent layers combine further. Because the projections can be reparameterized with later layers, individual components represent learned electrode combinations rather than uniquely identified neural sources. Figure~\ref{fig:augmentations} illustrates the training transformations.

\begin{figure}[!htbp]
\centering
\includegraphics[width=\textwidth]{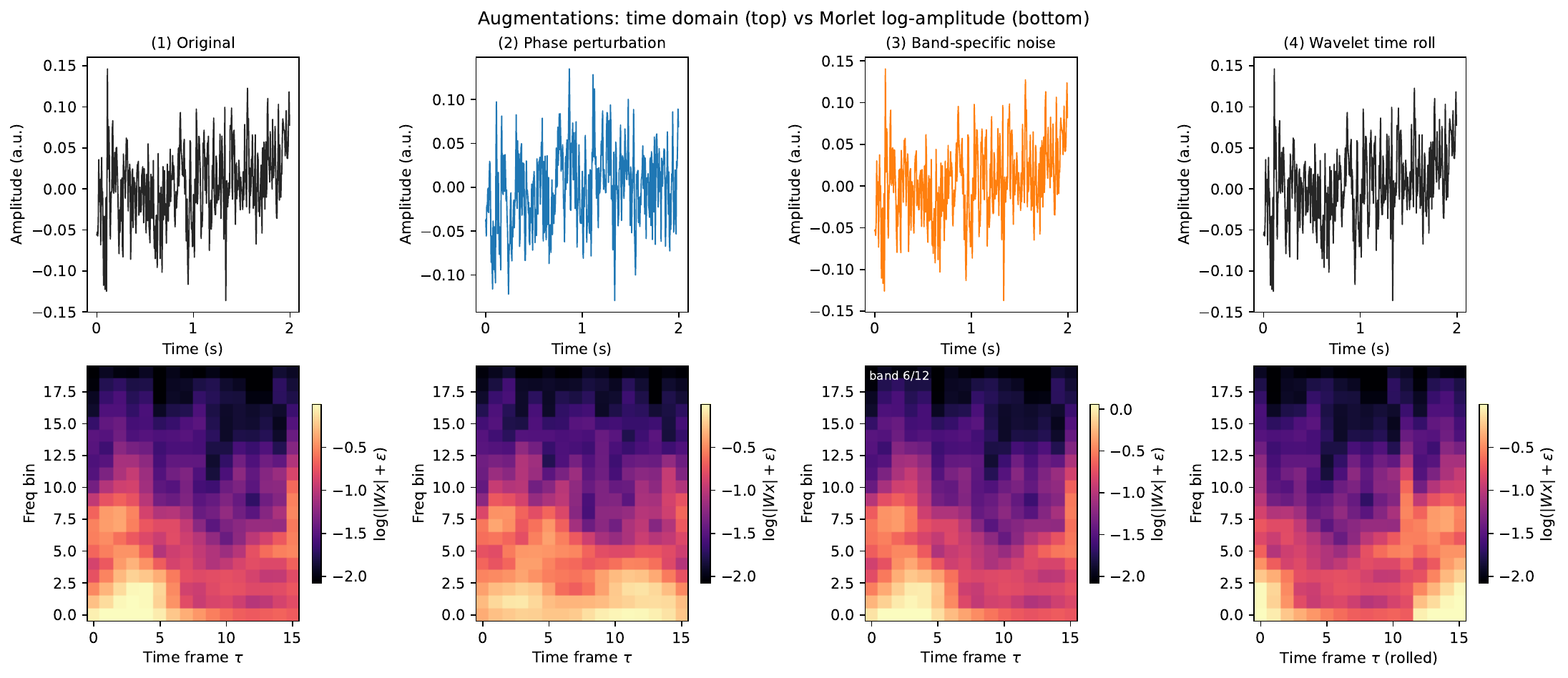}
\caption{Illustration of phase perturbation, band-specific noise, and wavelet time roll. Channel dropout is described in Appendix~\ref{app:datasets}.}
\label{fig:augmentations}
\end{figure}

\clearpage
\section{Extended Comparison with Related Methods}
\label{app:closest-work}
The contribution is the joint construction of continuous time-frequency tokens from a same-trial reference, an absolute/residual pair, and frequency-specific electrode mixing. Table~\ref{tab:closest-work} distinguishes this combination from related designs. JDAT jointly mixes channels and frequencies after wavelet decomposition, whereas MST projects electrodes independently at each sampled frequency. Ahmed et al.\ use a dedicated resting recording for spectral normalization. DMNet compares spectra within a single channel and passes their differences to a CNN. TFM-Tokenizer learns discrete motifs from raw and STFT views of individual channels. IAG is a task-specific precedent for learning spatial structure from bandwise emotion features.

\begin{table}[!htbp]
\centering
\small
\caption{Related representation designs and their published evaluation protocols.}
\label{tab:closest-work}
\begin{tabular}{@{}p{0.14\textwidth}p{0.43\textwidth}p{0.37\textwidth}@{}}
\toprule
Method & Representation and reference & Published evaluation \\
\midrule
JDAT \citep{wang2021jdat} & Pre-trial waveform subtraction on DEAP/DREAMER; corrected-wavelet representation with joint channel--frequency mixing. & Within-subject ten-fold evaluation; no baseline subtraction on SEED. \\
\addlinespace
Ahmed et al.\ \citep{ahmed2023baseline} & Trial-to-resting FFT ratio followed by band statistics; no parallel absolute branch. & Includes DEAP leave-one-participant-out evaluation using each target participant's resting reference. \\
\addlinespace
DMNet \citep{tu2024dmnet} & Single-channel spectral difference matrices from adjacent segments; clinical variant adds channel-level reference prototypes. & Subject-disjoint seizure detection; a different task and input from SEED emotion decoding. \\
\addlinespace
TFM-Tokenizer \citep{pradeepkumar2026tfm} & Raw and STFT single-channel views mapped to discrete motifs; no neighboring-window reference. & No SEED evaluation in the original paper; Table~\ref{tab:main} uses our local downstream adaptation. \\
\addlinespace
IAG \citep{song2020iag} & Instance-adaptive graphs over five-band DE features, followed by graph convolution and temporal modeling. & Includes SEED LOSO with one-second DE features; input and preprocessing differ from our two-second pipeline. \\
\bottomrule
\end{tabular}
\end{table}

EmT and Analformer illustrate two related tokenization strategies. EmT embeds temporal EEG feature graphs as a sequence of tokens and models them with a contextual Transformer for cross-subject emotion recognition \citep{ding2025emt}; MST instead averages neighboring spectra into a reference that conditions the target grid before attention. Analformer uses fixed Morlet kernels in analytical patch embedding before a Transformer, evaluated on motor imagery, ERP, and SSVEP \citep{yeom2026analformer}. MST combines this spectral encoding with a same-trial reference to condition each target token. These are architectural comparisons based on the published methods and protocols.

Our numerical comparisons use the general EEG encoders listed in Table~\ref{tab:main}, adapted to a common SEED input pipeline. Table~\ref{tab:closest-work} complements these experiments with a comparison of representation designs and published protocols. Ablations across all 15 SEED subjects assess the components of MST within this pipeline.

\end{document}